\documentclass{article}

\usepackage[preprint]{neurips_2026}

\usepackage[utf8]{inputenc} 
\usepackage[T1]{fontenc}    
\usepackage{hyperref}       
\usepackage{url}            
\usepackage{booktabs}       
\usepackage{amsmath,amsfonts,amsthm,amssymb}       
\usepackage{nicefrac}       
\usepackage{microtype}      
\usepackage{xcolor}         
\usepackage{algorithmic}
\usepackage[linesnumbered,ruled]{algorithm2e}

\usepackage{mathtools}
\usepackage{makecell}
\usepackage{dsfont}
\usepackage{enumitem}
\newtheorem{theorem}{Theorem}

\newtheorem{lemma}{Lemma}
\newtheorem{definition}{Definition}
\newtheorem{remark}{Remark}
\newtheorem{assump}{Assumption}

\newcommand{\ones}{\mathds{1}}

\newcommand{\cA}{\mathcal{A}}

\newcommand{\cF}{\mathcal{F}}

\newcommand{\cM}{\mathcal{M}}

\newcommand{\cO}{\mathcal{O}}

\newcommand{\inter}[1]{\ifblank{#1}{\inter}{\textnormal{int}(#1)}}

\newcommand{\tmix}{\tau_\textnormal{mix}}

\newcommand{\bE}{\mathbb{E}}

\newcommand{\bP}{\mathbb{P}}
\newcommand{\bR}{\mathbb{R}}

\newcommand{\cT}{\mathcal{T}}
\newcommand{\MLMC}{\textnormal{MLMC}}

\newcommand{\KL}{\textnormal{KL}}

\newcommand{\piS}{\pi^*}
\newcommand{\wS}{w^*}
\newcommand{\Tm}{T_\textnormal{max}}

\newcommand{\hbJ}{\hat{\bJ}^{\theta_k}}

\newcommand{\xS}{x^*}

\newcommand{\cL}{\mathcal{L}}
\newcommand{\cP}{\mathcal{P}}

\newcommand{\cS}{\mathcal{S}}

\newcommand{\bJ}{\mathbb{J}}

\newcommand{\bbeta}{\bar{\beta}}

\newcommand{\epsb}{\varepsilon_\textnormal{bias}}
\newcommand{\epsap}{\varepsilon_\textnormal{app}}

\usepackage{todonotes}

\newcommand{\xmark}{\ding{55}}

\newcommand{\floor}[1]{\left\lfloor #1 \right\rfloor}

\usepackage{pifont}
\usepackage[table]{xcolor}
\usepackage{todonotes}
\title{Bias-Controlled Primal-Dual Natural Actor-Critic: Optimal Rates for Constrained Multi-Objective Average-Reward RL}

\author{%
  Ankur Naskar\textsuperscript{\rm 1,2}, Swetha Ganesh\textsuperscript{\rm 2}, Vaneet Aggarwal\textsuperscript{\rm 2} \\
    \textsuperscript{\rm 1} Indian Institute of Science, Bengaluru, India\\
    \textsuperscript{\rm 2} Purdue University, USA
    \\
\texttt{ankurnaskar@iisc.ac.in, ganesh49@purdue.edu,
vaneet@purdue.edu}\\
}

\begin{document}

\maketitle

\begin{abstract}
  
Many reinforcement learning (RL) problems in the infinite-horizon average-reward setting require optimizing multiple conflicting objectives while satisfying multiple safety constraints. A common approach is concave scalarization, where the agent maximizes a utility $ f(J^\pi_{r_1}, \ldots, J^\pi_{r_M}) $ subject to a scalarized constraint $ g(J^\pi_{c_1}, \ldots, J^\pi_{c_N}) \ge 0 $, where $J^\pi_{r_m}$ and $J^\pi_{c_n}$ denote the average-reward and cost under policy $\pi$. However, the nonlinearity of $f$ and $g$ introduces bias in policy-gradient and actor-critic methods, since gradients must be evaluated using noisy estimates of $J^\pi,$ and $ \mathbb{E}[\partial f(X)] \neq \partial f(\mathbb{E}[X]),$ and this bias propagates through both primal and dual updates. We propose an MLMC-based primal-dual Natural Actor-Critic algorithm for average-reward MDPs that controls bias in scalarized objectives, constraint evaluation, and actor-critic estimation without requiring mixing-time knowledge. We show that the algorithm achieves optimal global convergence and constraint violation rates of $ \tilde{\mathcal{O}}(1/\sqrt{T}) $. To our knowledge, this is the first result establishing optimal convergence for concave scalarized multi-objective RL in the average-reward setting, both with and without constraints, and the first to do so without mixing-time information even in the absence of scalarization. 

\end{abstract}

\vspace{-.1in}
\section{Introduction}


Many real-world reinforcement learning (RL) systems operate over long horizons and must simultaneously optimize multiple competing objectives while satisfying safety constraints, as in network resource allocation, recommender systems, and autonomous control~\cite{dulac, zheng2023network, 10.1145/3543507.3583313,  kushwaha2026surveysafereinforcementlearning, lazic2018data, gao2024constraints, hayes2022practical}. In such settings, performance is naturally measured via long-term average-reward rather than discounted return, since discounting can bias short-term behavior and depend sensitively on the choice of discount factor. At the same time, these applications require balancing objectives such as efficiency, fairness, and risk under operational constraints. A principled approach to model such trade-offs is via concave scalarization, where the agent maximizes a utility $f(J^\pi_{r_1}, \ldots, J^\pi_{r_M})$ subject to a constraint $g(J^\pi_{c_1}, \ldots, J^\pi_{c_N}) \ge 0$, with $J^\pi_{r_m}$ and $J^\pi_{c_n}$ denoting long-run average-rewards and costs. Concave scalarization provides a flexible and expressive framework for multi-objective decision making. However, solving such problems in large-scale settings requires algorithms that can handle function approximation and continuous action spaces.

Policy-gradient and actor--critic methods provide a scalable approach for average-reward RL~\citep{wei2020model, kumar2025global, chen2023finite, wang2024non, tian2023convergence, murthy2023modified}. Recent work has established optimal $\tilde{\cO}(1/\sqrt{T})$ convergence rates for parametrized policies in both unconstrained and constrained single-objective settings \citep{pmlr-v267-ganesh25b, xu2026global, pmlr-v202-suttle23a}. Extending these methods to concave scalarized multi-objective problems introduces a fundamental challenge: the nonlinearity of $f$ and $g$ induces bias in gradient estimation since $\bE[\partial f(X)] \neq \partial f(\bE[X])$. In the average-reward setting, this bias interacts with stationary distribution estimation and propagates through both actor-critic updates and primal-dual dynamics, making convergence analysis significantly more challenging.

\begin{table}[t]
\centering
\small
\setlength{\tabcolsep}{4pt}
\renewcommand{\arraystretch}{1.1}
\begin{tabular}{l l c c c c c}
\toprule
\textbf{Setting} & \textbf{Method} & \textbf{Param.} & \textbf{Constr.} & \textbf{Concave} & \textbf{Mixing} & \textbf{Rate} \\
\midrule

Avg & Actor-Critic & \checkmark & \xmark & \xmark & No & $\tilde{\mathcal{O}}(1/\sqrt{T})$  \citep{pmlr-v267-ganesh25b} \\

Avg & Primal-Dual AC & \checkmark & \checkmark & \xmark & Yes & $\tilde{\mathcal{O}}(1/\sqrt{T})$  \citep{xu2026global} \\

Avg & Model-based & \xmark & \xmark & \checkmark & No & $\tilde{\mathcal{O}}(1/\sqrt{T})$ \citep{agarwal2022multi,agarwal2023reinforcement} \\

Avg & Model-based PD & \xmark & \checkmark & \checkmark & No & $\tilde{\mathcal{O}}(1/\sqrt{T})$ \cite{agarwal2022concave} \\

\midrule

Disc & Policy Gradient & \checkmark & \xmark & \checkmark & N/A & $\tilde{\mathcal{O}}(\epsilon^{-2})$  \cite{ganesh2026breakingbiasbarrierconcave} \\

\midrule

\rowcolor{gray!10}
\textbf{Avg} & \textbf{Primal-Dual AC} & \checkmark & \checkmark & \checkmark & \textbf{No} & $\tilde{\mathcal{O}}(1/\sqrt{T})$ \\
\rowcolor{gray!10}
\multicolumn{7}{c}{\textbf{(This work)}} \\

\bottomrule
\end{tabular}
\caption{\small This table summarizes the key related work. The columns indicate whether the setting is average or discounted reward, whether the policy is parametrized, whether constraints are present, whether concave scalarization is used for multiple objectives, whether the algorithm requires mixing-time knowledge, and the corresponding convergence rate or sample complexity.  In the average-reward setting, existing methods handle either parametrized policies or concave multi-objective formulations, but not both simultaneously under constraints without mixing-time assumptions. AC=Actor-Critic, PD=Primal-Dual. Detailed related work is discussed in Appendix \ref{apd:related}. }
\label{tab:comparison}
\vspace{-.3in}
\end{table}
Existing works address only partial aspects of this problem. Model-based approaches handle concave scalarized multi-objective RL in the average-reward setting \citep{ agarwal2022multi, agarwal2023reinforcement, agarwal2022concave, zahavy2021reward}, but do not scale to parametrized policies. Conversely, actor-critic methods achieve optimal rates for parametrized policies in single-objective settings \citep{pmlr-v267-ganesh25b, xu2026global}, but do not handle scalarized multi-objective formulations. Recent progress on bias control for concave scalarization applies only to discounted MDPs \citep{ganesh2026breakingbiasbarrierconcave}, while has been studied only in the unconstrained case. As summarized in Table~\ref{tab:comparison}, no prior work establishes finite-time guarantees for concave scalarized multi-objective RL with parametrized policies in the average-reward setting. Moreover, even in the discounted setting, such guarantees are available only for the unconstrained case.

This gap stems from the simultaneous presence of three coupled challenges: (i) bias induced by concave scalarization, (ii) approximation bias in actor-critic methods under function approximation, and (iii) amplification of these errors through primal-dual updates in the average-reward setting. Existing techniques address these challenges in isolation, but not jointly. In this paper, we resolve this gap by developing a Multi-Level Monte Carlo (MLMC)-based primal-dual Natural Actor-Critic algorithm for constrained multi-objective average-reward MDPs. Our approach jointly controls the coupled sources of bias while avoiding any dependence on the mixing-time. The main contributions are summarized as follows:
\begin{itemize}[leftmargin=0pt,labelindent=0pt,labelwidth=!,itemindent=6pt,align=left,itemsep=2pt,topsep=2pt,parsep=0pt]

\item \textbf{Bias-controlled primal-dual actor-critic for multi-objective average-reward RL.} We develop a primal-dual Natural Actor-Critic algorithm that handles concave scalarized objectives and constraints for parametrized policies, while explicitly controlling bias arising from scalarization, actor-critic estimation, and dual updates.

\item \textbf{Optimal convergence under coupled bias.} We show in Theorem~\ref{thm: main result} that our algorithm achieves optimal global convergence and constraint violation rates of $\tilde{\cO}(1/\sqrt{T})$, despite the presence of bias induced by nonlinear scalarization and function approximation.

\item \textbf{No mixing-time dependence.} Our MLMC-based design eliminates the need for mixing-time knowledge, improving upon prior primal-dual actor-critic methods. While \cite{xu2020improving} relies on mixing-time knowledge for optimal convergence in the special case of no scalarization and single objective, our prescribed choice of parameters removes any mixing-time dependence even in this special case. 

\item \textbf{First finite-time guarantees for parametrized concave scalarized average-reward RL.} To our knowledge, this is the first work to establish finite-time guarantees for concave scalarized multi-objective RL with parametrized policies in the average-reward setting, both with and without constraints.

\end{itemize}

\section{Related works}\label{apd:related}

\paragraph{Average-Reward Constrained RL.}
Constrained Markov decision processes (CMDPs) have been extensively studied in the tabular setting \citep{chen2022learning, gattami2021reinforcement, agarwal2022concave, agarwal2022regret}. More recently, these problems have been analyzed under parametrized policies \citep{bai2024learning, xu2026global}. In particular, \citep{xu2026global} established the optimal $\tilde{\cO}(1/\sqrt{T})$ convergence rate using a primal-dual Natural Actor-Critic method when the mixing-time is known. In the absence of mixing-time knowledge, their analysis yields a suboptimal rate of $\tilde{\cO}(1/T^{1/2-\epsilon})$ for $T > t_{\mathrm{mix}}^{2/\epsilon}$.

In this work, we consider a more general setting with multiple objectives combined via concave scalarization. As a byproduct of our analysis, even in the single-objective case without scalarization, we improve upon \citep{xu2026global} by achieving the optimal $\tilde{\cO}(1/\sqrt{T})$ rate without requiring mixing-time knowledge.

\paragraph{Average-Reward Concave Scalarized Multi-Objective RL.}
The problem of average-reward concave scalarized multi-objective RL has been studied in the tabular setting \citep{agarwal2022multi, agarwal2023reinforcement}, and extended to the constrained case in \citep{agarwal2022concave}. However, these approaches rely on model-based or tabular representations and do not scale to large state or action spaces. In contrast, handling parametrized policies is essential for scalability, yet finite-time guarantees for concave scalarized multi-objective RL with parametrized policies in the average-reward setting remain largely unexplored. Addressing this gap is a central focus of this paper.

\paragraph{Discounted-Reward Concave Scalarized Multi-Objective RL.}
In the discounted setting, \cite{10.1613/jair.1.13981} proposed a model-free policy-gradient algorithm for concave multi-objective RL and established a sample complexity of $\cO(\epsilon^{-4})$ for computing an $\epsilon$-optimal policy. Subsequent work \cite{zhou2022anchorchanging} analyzed natural policy-gradient methods for concave utility RL under access to exact gradients, thereby bypassing the sampling challenges inherent in practical settings. More recently, \cite{ganesh2026breakingbiasbarrierconcave} explicitly addressed the bias induced by concave scalarization and obtained an improved sample complexity of $\cO(\epsilon^{-2})$.

However, these approaches rely critically on properties specific to discounted MDPs, such as geometric mixing and finite-horizon truncation, and do not extend directly to the average-reward setting, where estimation depends on stationary distributions. Moreover, none of these works consider constraints. Incorporating constraints through a primal-dual framework introduces additional challenges, as the lack of strong convexity under parametrization and the coupling between primal and dual updates amplify the impact of scalarization-induced bias, requiring more careful analysis.

\vspace{-.1in}
\section{Problem setup}

Consider a multi-objective average-reward constrained Markov Decision Process (CMDP) $\cM= (\cS,\cA,\cP, \{r_m\}_{m=1}^{M}, \{c_n\}_{n=1}^{N}, \rho),$ where $\cS$ and $\cA$ are the finite state and action spaces, $\cP:\cS\times\cA\to\Delta(\cS)$ is the transition kernel, $r_m:\cS\times\cA\to [0,1]$ denotes the reward function associated with objective $m\in [M],$ $c_n:\cS\times\cA\to [-1,1]$ denotes the cost associated with constraint $n\in[N],$ and $\rho\in \Delta(\cS)$ is an initial distribution over $\cS.$ Let $\pi:\cS\to \Delta(\cA),$ be a stationary policy mapping each state to a distribution over actions. For each objective $m\in [M]$ and constraint $n\in[N],$ define the average-reward $J^\pi_{r_m}$ and average-constraint cost $J^\pi_{c_n},$ associated with $\pi$ as
\begin{equation}\label{e: def.avg-rew.avg-cost}
    J^\pi_{u} := \lim_{T\to\infty} \frac{1}{T}\bE_\pi\bigg[ \sum_{t=0}^{T-1}  u(s_t,a_t) \bigg| s_0\sim \rho\bigg], \ u\in \{r_m, c_n\}.
\end{equation}
When dealing with large state spaces, the common practice is to consider parametrized policies $\{\pi_\theta\},$ with $\theta\in \Theta\subset \bR^d$. Let us denote $J^{\pi_\theta}_u$ as $J^{\theta}_u,$ for any $u \in \{r_m,c_n: m\in[M],n\in{N} \},$ and let $\bJ^\theta_r := (J^\theta_{r_1},\ldots, J^\theta_{r_M})$ and $\bJ^\theta_c := (J^\theta_{c_1},\ldots, J^\theta_{c_N}).$

To ensure $J^\theta_u$ is well-defined, we assume the following:
\begin{assump}[Ergodicity]\label{a: ergodicity}
    The CMDP $\cM$ is ergodic, i.e., for every policy $\pi,$ the Markov chain induced by $\pi$ and $\cP$ on $\cS$ is irreducible and aperiodic.
\end{assump}
Before moving forward, we introduce the notion of mixing-time, which measures how quickly any Markovian trajectory of the MDP reaches a stationary distribution.
\begin{definition}[Mixing-time]
    Given a policy parameter $\theta,$ we define mixing-time of the CMDP $\cM$ as $\tmix^\theta:= \min \Big\{ t>0 : \max_{s\in\cS} \left\| (\cP^\theta)^t(s,\cdot) - d^\theta(\cdot) \right\| \leq 1/4 \Big\}.$ We also define $\tmix:= \sup_{\theta\in\Theta}\tmix^\theta$ as the overall mixing-time of $\cM.$
\end{definition}
\begin{remark}
    Assumption~\ref{a: ergodicity} implies that, for every $\theta,$ there exists a $\rho$-independent stationary distribution $d^\theta$ such that $d^\theta \cP^\theta = d^\theta,$ where $\cP^\theta(s,s'):= \sum_a \pi_\theta(a|s)\cP(s'|s,a).$ Moreover, $J^\theta_u = \bE_{\nu^\theta}[u(s,a)],$ for any $u=r_m, c_n,$ $m\in[M], n\in [N],$ with $\nu^\theta(s,a):= d^\theta(s)\ \pi_\theta(a|s).$ Furthermore, $\tmix$ is finite under this assumption.
\end{remark}

Our goal in this paper is to obtain a policy that solves the following constrained optimization
\begin{equation}\label{e: goal.optimize}
    \max_{\theta} f(\bJ^\theta_r)\ \text{ s.t } \ g(\bJ^\theta_c) \geq 0, 
\end{equation}
where $f:\bR^M\to \bR$ and $g:\bR^N\to\bR$ are concave scalarization functions \footnote{Note that despite the concavity of the scalarization $f$ and $g,$ the non-concavity of $\bJ^\theta_u,$ $u\in\{r,c\},$ renders the problem \eqref{e: goal.optimize} non-concave in $\theta.$}. 

To facilitate our analysis, we make the following assumptions. 
\begin{assump}[Slater condition]\label{a: slater.feasibility}
    There exists $\delta>0$ and $\bar{\theta} \in \Theta$ such that $g(\bJ^{\bar{\theta}_c}) > \delta.$
\end{assump}
\begin{remark}
    Assumption~\ref{a: slater.feasibility} is common in the CMDP literature~\citep{xu2026global, wei2022provably}, and ensures that \eqref{e: goal.optimize} admits a feasible solution.
\end{remark}
\begin{assump}[Smoothness of $f$ and $g$]\label{a: smoothness}
    The scalarizations $f$ and $g$ are $L_f$-smooth on $[0,1]^M$ and $L_g$-smooth on $[-1,1]^N$, respectively,  i.e., for each objective $m$ and constraint $n,$
    \[
        \left| \partial_m f(x) - \partial_m f(y)  \right| \leq L_f\| x - y \|_2, 
        \quad  \text{and}  \quad 
        \left| \partial_n g(x) - \partial_n g(y)  \right| \leq L_g\| x - y \|_2, 
    \]
    where $x,y\in [0,1]^M$ in the first inequality, whereas in the second inequality, $x,y \in [-1,1]^N.$ 
\end{assump}
\begin{remark}\label{rem: smooth}
    Since the rewards and constraint costs lie in $[0,1]$ and $[-1,1],$ respectively, it suffices to assume smoothness of $f$ and $g$ on $[0,1]^M$ and $[-1,1]^N,$ respectively. Further, since these sets are compact, $\partial_m f$ and $\partial_n g$ are uniformly bounded on these sets, by some constant $C>0.$
\end{remark}

\paragraph{Primal-dual Natural Policy Gradient:} We tackle  \eqref{e: goal.optimize} using the primal-dual approach by solving the following saddle point problem.
\begin{equation}\label{e: lagrange.function}
    \max_{\theta\in\Theta} \min_{\lambda\geq 0} \cL(\theta, \lambda) := f(\bJ^\theta_r) + \lambda g(\bJ^\theta_c),
\end{equation}
where $\cL(\cdot,\cdot)$ and $\lambda$ are called the Lagrange function and the Lagrange multiplier, respectively. In unconstrained single-objective RL ($f=\text{id}$ with $m=1$ and $\lambda=0$), a widely used approach for solving this optimization is via the Natural Policy Gradient (NPG) algorithm \citep{xu2026global}. For our setup, the extension becomes:
\begin{equation}\label{e: primal-dual.NPG}
    \theta_{k+1} = \theta_k + \alpha F(\theta_k)^\dagger\nabla_\theta \cL(\theta_k, \lambda_k), \quad \lambda_{k+1} = \Pi_{[0,2/\delta]}\left(\lambda_k - \beta g(\bJ^{\theta_k}_c) \right),
\end{equation}
where $\alpha$ and $\beta$ are the primal and dual stepsizes, respectively, $\delta$ is the constant from Assumption~\ref{a: slater.feasibility},  $\dagger$ denotes the Moore-Penrose pseudoinverse, and $F(\theta)$ is the Fisher information matrix defined as 
\begin{equation}\label{e: Fisher.matrix}
    F(\theta) := \bE_{(s,a)\sim \nu^\theta}\left[ \nabla_\theta \ln \pi_\theta(a|s) \ \left(\nabla_\theta \ln \pi_\theta(a|s)\right)^\top \right],
\end{equation}
while for all $A\subseteq \bR,$ we let $\Pi_A$ denote the projection onto $A.$ From \eqref{e: lagrange.function}, we have  $\nabla_\theta \cL(\theta, \lambda) = \nabla_\theta f(\bJ^\theta_r) + \lambda \nabla_\theta g(\bJ^\theta_c),$ where
\begin{equation}\label{e: gradients.f.g}
    \nabla_\theta f(\bJ^\theta_r) = \sum_{m=1}^{M} \partial_m f(\bJ^\theta_r)\nabla_\theta J^\theta_{r_m} \quad \text{and} \quad \nabla_\theta g(\bJ^\theta_c) = \sum_{n=1}^{N} \partial_n g(\bJ^\theta_c)\nabla_\theta J^\theta_{c_n}.
\end{equation}
For objective $m,$ and constraint $n,$ the gradients $\nabla_\theta J^\theta_{r_m}$ and $\nabla_\theta J^\theta_{c_n}$ are given by the policy gradient theorem~\cite[Eq.(6)]{pmlr-v267-ganesh25b} as
\begin{equation}\label{e: policy.gradient}
    \nabla_\theta J^\theta_u = \bE_{(s,a)\sim \nu^\theta}\left[ A^\theta_u(s,a)\nabla_\theta \ln \pi_\theta(a|s) \right],\ u\in\{ r_m, c_n \},
\end{equation}
where $A^\theta_u$ is the advantage function w.r.t the reward/cost $u$ and is defined as 
\begin{equation}\label{e: advantage.func}
    A^\theta_u(s,a) := u(s,a) - J^\theta_u + \bE_{s'\sim \cP(\cdot|s,a)}\left[V^\theta_u(s')\right] - V^\theta_u(s),
\end{equation}
with $V^\theta_u$ denoting the value function w.r.t $u,$ which is the unique (modulo $\textnormal{span}\{\ones\}$) solution in $\bR^\cS$ to the Bellman equation $V = \bE_{a\sim \pi_\theta(\cdot|s)}\left[u(s,a)\right] - J^\theta_u + \bE_{s'\sim \cP^\theta(s,\cdot)}\left[ V(s') \right].$

In section~\ref{s: algoritm}, we describe our algorithm that estimates the advantage function $A^\theta_u,$ the policy gradient $\nabla_\theta J^\theta_u,$ and the NPG vector $F(\theta)^\dagger\nabla_\theta \cL(\theta,\lambda)$ and updates the policy and Lagrange parameter pair $(\theta,\lambda).$

\vspace{-.1in}
\section{Algorithm}
\label{s: algoritm}

In this section, we present our MLMC-based multi-objective primal-dual natural actor-critic algorithm (MO-PDNAC), with pseudocode given in Algorithm~\ref{alg: MOPDNAC} in the Appendix~\ref{appendix: algorithm}. 

The algorithm runs for $K$ iterations (\textit{Outer loop}), updating the parameter pair $(\theta_k, \lambda_k)$ via the following rule:
\begin{equation}\label{e: actor.update}
    \theta_{k+1} = \theta_k + \alpha w_k, \quad \lambda_{k+1} = \Pi_{[0,2/\delta]}\left(\lambda_k - \beta g_k\right),
\end{equation}
where $w_k$ and $g_k$ are estimates of the NPG vector $F(\theta_k)^\dagger\nabla_\theta\cL(\theta_k, \lambda_k),$ and the constraint cost $g(\hbJ_c),$ respectively. 

Next, we describe the estimation procedure for $(w_k, g_k)$ via three key subroutines.

\paragraph{Critic Estimation:} Given current policy $\pi_{\theta_k}$, the critic is captured by  $\{(V^{\theta_k}_u, J^{\theta_k}_u)\}_{u\in\{r,c\}}.$ Adopting the strategy used in \citep{xu2026global}, we approximate  $V^{\theta_k}$ by $\Phi\zeta^{\theta_k}_u,$ where   $\Phi = (\phi(s):s\in \cS)^\top,$ such that each $\phi(s)\in\bR^d$ is a feature vector with $d\ll|\cS|.$ Further, fix a parameter $c_\gamma>0.$ Then, the vector $\xi^*_{u,k}:= [J^{\theta_k}_u, \zeta^{\theta_k}_u]^\top$ satisfies $P(\theta_k)\ \xi^*_{u,k} = q_u(\theta_k),$ where
\begin{align}\label{e: critic.estimation.matrix}
    q_u(\theta):= \bE_{(s,a)\sim \nu^\theta}\begin{bmatrix} c_\gamma u(s,a) \\ u(s,a)\phi(s) \end{bmatrix}, 
    \quad  
   P(\theta) := \bE_{\substack{s\in d^\theta\\ s'\sim \cP^\theta(\cdot|s)}}
    \begin{bmatrix}
       c_\gamma & 0 \\ \phi(s) & \phi(s)(\phi(s) - \phi(s'))^\top
    \end{bmatrix}.
\end{align}
%
Given $H$ trajectories $\{\cT^{kh}\}_{h=0}^{H-1} $ and noisy estimates $(\hat{q}^{kh}_u,\hat{P}^{kh})_{h=0}^{H-1}$ for $(q(\theta_k), P(\theta_k)),$ we update $(\xi^k_{k,h})$ using the stochastic linear recursion
\begin{equation}\label{e: critic.joint.update}
    \xi^k_{u,h+1} = \xi^k_{u,h} + \gamma_\xi\big[\hat{q}^{kh}_{u} - \hat{P}^{kh}\ \xi^k_{u,h} \big].
\end{equation}
After $H$ iterations (\textit{Inner loops}), we set $\xi_{u,k}=\xi^{k}_{u,H}$ as the critic estimate for policy $\pi_{\theta_k}.$



Next, we describe the subroutine that estimates the NPG direction.

\paragraph{NPG Estimation:} Note that the NPG vector $\wS_k:= F(\theta_k)^\dagger\nabla_\theta\cL(\theta_k, \lambda_k)$ is a minimizer of the function

$\qquad\quad  L_{\nu^{\theta_k}}(\theta_k, \lambda_k, w) := \frac{1}{2} w^\top F(\theta_k)w
- w^\top \left( \nabla_\theta f(\bJ^{\theta_k}_r) + \lambda_k\nabla_\theta g(\bJ^{\theta_k}_c)\right),$

with gradient
\begin{multline}
    \nabla_w L_{\nu^{\theta_k}}(\theta_k, \lambda_k, w) 
    := \bE_{(s,a)\sim \nu^{\theta_k}}\bigg[\nabla_\theta\ln \pi_{\theta_k}(a|s)\left(\nabla_\theta\ln \pi_{\theta_k}(a|s)\right)^\top w  \\
    -\nabla_\theta\ln \pi_{\theta_k}(a|s)\bigg(\sum_{m=1}^M\partial_mf(\bJ^{\theta_k}_{r})A^{\theta_k}_{r_m}(s,a) + \lambda_k\sum_{n=1}^N\partial_n g(\bJ^{\theta_k}_{c})A^{\theta_k}_{c_n}(s,a)\bigg)  \bigg].
\end{multline}
This motivates us to employ a stochastic gradient descent (SGD) scheme. 

Given critic estimates $\xi_{u,k} = [\eta_{u,k}, \zeta_{u,k}]^\top,$ $u\in\{r_m,c_n: m\in[M],n\in[N]\},$ and scalar gradient estimates $\{\partial_m f(\hat{\bJ}^{\theta_k}_r), \partial_n g(\hat{\bJ}^{\theta_k}_c): m\in[M],n\in[N]\},$ along with $H$ trajectories $\{\cT^{kh}\}_{h=0}^{H-1}$ and noisy estimates $(\hat{F}^{kh}, \hat{b}^{kh}_u)_{h=0}^{H-1}$ for the $(F(\theta_k), A^{\theta_k}_u),$ we run the following linear recursion.
\begin{align}\label{e: NPG.update}
    w^k_{h+1} = w^k_h - \gamma_w\hat{\nabla}_w L_{\nu^{\theta_k}}(\theta_k,\lambda_k, w^k_h),
\end{align}
where we define the stochastic gradient as
\begin{align}\label{e: stochastic.grad}
    \hat{\nabla}_wL_{\nu^{\theta_k}}(\theta_k,\lambda_k, w^k_h) := \hat{F}^{kh}\ w^k_h - \sum_{m=1}^{M}\partial_m f(\hat{\bJ}^{\theta_k}_r)\hat{b}^{kh}_{r_m} - \lambda_k\sum_{n =1}^{N}\partial_n g(\hat{\bJ}^{\theta_k}_c)\hat{b}^{kh}_{c_n}.
\end{align}
After $H$ iterations, we set $w_k:=w^k_H,$ giving us the estimate for the NPG vector. 

In the following subroutine, we use an MLMC-based scheme to obtain the noisy constraint-cost $g_k,$ the scalar gradients, and the matrix estimates used in \eqref{e: actor.update}, \eqref{e: critic.joint.update} and~\ref{e: NPG.update}.

\paragraph{MLMC-Based Estimation:} Throughout the algorithm, we fix  $\Tm>0$ as the maximum trajectory length. First, we elaborate the critic and NPG estimators. 

Given a trajectory $\cT^{kh}:= ((s^{kh}_t,a^{kh}_t, s^{kh}_{t+1}))_{t=0}^{\ell_{kh}-1},$ with $\ell_{kh}=2^{Q_{kh}\ones_{\{ 2^{Q_{kh}}\leq \Tm\}}},$ generated using policy $\pi_{\theta_k},$  we define for each $X\in\{ \hat{q}^{kh}_u, \hat{b}^{kh}_u, \hat{F}^{kh}, \hat{P}^{kh} \},$ the naive estimator with $T$ samples as
$X_T := \frac{1}{T}\sum_{t=0}^{T-1}X(s^{kh}_t, a^{kh}_t, s^{kh}_{t+1})$
where
\begin{equation}\label{e: critic.naive.single.step}
    \hat{q}_u(s,a) :=\begin{bmatrix} c_\gamma u(s,a)\\ u(s,a)\phi(s) \end{bmatrix}, \quad \hat{P}(s,a,s') := \begin{bmatrix} c_\gamma  & 0\\ \phi(s) & \phi(s)(\phi(s) - \phi(s'))^\top \end{bmatrix},
\end{equation}
and
\begin{multline}\label{e: NPG.naive.single.step}
    \hat{F}^{kh}(s,a): = \nabla_\theta \ln\pi_{\theta_k}(a|s)\ \left( \nabla_\theta\ln\pi_{\theta_k}(a|s)\right)^\top
    \\
    \hat{b}^{kh}_u(s,a,s') := \left( u(s,a) - \eta_{u,k} + [\phi(s') - \phi(s)]^\top\zeta_{u,k} \right)\cdot\nabla_\theta\ln \pi_{\theta_k}(a|s).
\end{multline}

The corresponding MLMC estimator is defined as:
\begin{align}\label{e: MLMC}
    X_\MLMC := X_1 + \ones_{\{2^{Q_{kh}} \leq \Tm\}}\  2^{Q_{kh}}\left( X_{2^{Q_{kh}}} - X_{2^{Q_{kh}-1}} \right), \ \text{where} \quad Q_{kh} \sim \textnormal{Geom}(1/2).
\end{align}

Now, we describe the constraint-cost and scalar gradient estimators. 

Given trajectory $\cT^{k}:= (s^k_t,a^k_t,s^k_{t+1})$ generated using $\pi_{\theta_k},$ define the naive estimator with $T$ samples as $ g(\hat{\bJ}^{\theta_k}_{c,T})$ for the constraint-cost, and $\partial_m f(\hat{\bJ}^{\theta_k}_{r,T})$ and $\partial_n g(\hat{\bJ}^{\theta_k}_{c,T}),$ for scalar gradients, respectively, where $\hat{\bJ}^{\theta_k}_{u,T} := \frac{1}{T}\sum_{t=0}^{T-1}u(s^k_t,a^k_t),$ for vector-valued $u\in\{r,c\}.$ Then, for each $X\in \{g, \partial_m f, \partial_n g\},$ the corresponding MLMC estimators are defined as
\begin{align}\label{e: scalar.grad.MLMC}
    X(\hat{\bJ}^{\theta_k}_{u,\MLMC}) & := X(\hat{\bJ}^{\theta_k}_{u,1}) + \ones_{\{2^{Q_k}\leq\Tm\}}\ 2^{Q_k}\left( X(\hat{\bJ}^{\theta_k}_{u,2^{Q_k}}) - X(\hat{\bJ}^{\theta_k}_{u,2^{Q_k-1}}) \right),
\end{align}

%
where $Q_k\sim \textnormal{Geom}(1/2).$ Note that, in the above expression $u=r$ for $X=\partial_m f,$ and $u=c$ for $X=g$ and $X=\partial_n g.$ 

\begin{remark}
    Notice that the above MLMC-based estimators require the trajectory length to be sampled from $\textnormal{Geom}(1/2),$ and do not require knowing the underlying MDP's mixing-time.
\end{remark}
\begin{remark}
    The MLMC estimators in~\eqref{e: scalar.grad.MLMC} and~\eqref{e: MLMC} give the same order of bias as the naive estimators with $\Tm$ samples, but only use $O(\log_2 \Tm)$ expected number of samples, reducing the overall sample complexity by a factor of $\tilde{\Omega}(\Tm).$ While the variance is worse for MLMC estimators, it does not alter the final $\tilde{\cO}(1/\sqrt{T})$ convergence rate.
\end{remark}

\vspace{-.1in}
\section{Main results}
\label{s: main.results}

In this section, we present our finite-time regret bounds in Theorem~\ref{thm: main result}. In this paper, we make the following assumptions.


\begin{assump}\label{a: score.function}
    There exist $G_1,G_2>0,$ such that for all $\theta\in\Theta,$ and state-action pair $(s,a),$ 
    \[
        \|\nabla_\theta\ln\pi_\theta(a|s)\| \leq G_1 \quad \text{and}\quad \|\nabla_\theta\ln\pi_\theta(a|s) - \nabla_\theta\ln\pi_{\theta'}(a|s)\| \leq G_2\|\theta-\theta'\|.
    \]
\end{assump}

\begin{assump}\label{a: Fisher.degenrate}
    There exists $\mu_F>0,$ such that for all $\theta\in \Theta,$ $F(\theta) \succeq \mu_F I.$ 
\end{assump}

\begin{assump}\label{a: FA.error.PG}
    Let $\piS$ be the optimal policy, $\nu^*$ be the stationary distribution over $\cS\times\cA$ induced by $\piS.$ Then there exists $\epsb>0$ such that for every $\theta\in \Theta$ and $\lambda\in [0,2/\delta],$ we have 
    \[
        \min_{w\in \bR^d} L_{\nu^*}(\theta,\lambda,w)\leq \epsb.
    \]
\end{assump}

\begin{assump}\label{a: feature.degenrate}
    The feature matrix $\Phi$ satisfies $\|\Phi\|\leq 1.$ Further, there exists $\mu_\phi>0$ such that 
    \[
        \bE_{s\sim d^\theta,s'\sim \cP^\theta(\cdot|s)}\left[\phi(s)(\phi(s) - \phi(s'))^\top \right] \succeq \mu_\phi I.
    \]
\end{assump}

\begin{assump}\label{a: FA.error.value}
   There exists $\epsap>0$ such that for every $u\in \{r_m,c_n: m\in [M], n\in [N]\},$ we have
    \[
        \min_{\zeta\in \bR^d}\frac{1}{2}\bE_{s\sim d^\theta}\left[\left\|V_u^{\theta}(s) - \Phi \zeta\right\|^2\right] \leq \epsap.
    \]
\end{assump}

\begin{remark}[\textbf{\textit{Policy gradient}}]
    Assumption~\ref{a: score.function} requires the score function to be bounded and Lipschitz, and allows us to establish an improvement in the objective function resulting from \eqref{e: actor.update}. Assumptions~\ref{a: score.function}--\ref{a: FA.error.PG} are standard in the PG literature\cite{liu2020improved, papini2018stochastic, Xu2020Sample, fatkhullin2023stochastic, ding2025convergence}. Assumption~\ref{a: Fisher.degenrate} ensures that the Fisher matrix is invertible,  guaranteeing that the SGD iteration for finding the NPG direction is stable and converges to a unique limit. Assumption~\ref{a: FA.error.PG} quantifies the quality of the policy parametrization; smaller $\epsb$ implies richer parametrization.
\end{remark}
\begin{remark}[\textbf{\textit{Critic estimation}}]
    Assumptions~\ref{a: feature.degenrate} and~\ref{a: FA.error.value} are standard in actor-critic and policy evaluation methods \cite{panda2025two, wu2020finite, zhang2021finite}. Assumption~\ref{a: feature.degenrate} guarantees stability and convergence for the critic subroutine to a unique limit; it suffices to choose features such that the all ones vector does not lie in the column space of $\Phi.$ Assumption~\ref{a: FA.error.value} quantifies the function-approximation capacity of the features.
\end{remark}


We now present our main result in the following Theorem.

\begin{theorem}[\textbf{\textit{Global convergence and constraint violation}}]\label{thm: main result}
    Assume that each $J^\theta_u$ is $L$-smooth in $\theta$ and let $(\theta_k,\lambda_k)$ be updated by~\eqref{e: actor.update}. Further,  let $\piS$ be optimal across all stationary policies \footnote{$\piS$ is a solution to the constrained optimization problem: $\max_{\pi:\cS\to\Delta(\cA)}f(\bJ^{\pi}_r)$ s.t. $g(\bJ^{\pi}_c)\geq 0$}, let Assumptions~\ref{a: ergodicity}--\ref{a: FA.error.value} hold. Choose $\gamma_w = \frac{2\ln T}{\mu_F H},$ $\gamma_\xi = \frac{2\ln T}{\mu_\phi H},$ $\alpha = \beta = \frac{1}{\sqrt{T}},$ and set $\Tm = K = T,$ and $H=\Theta(\ln T).$  Then, we have 
    \begin{align}\label{e: opt.converge}
        f(\bJ^{\piS}_r) - \frac{1}{K}\sum_{k=0}^{K-1} \bE\left[ f(\bJ^{\theta_k}_r)\right] & = \tilde{\cO}\left(\frac{1}{\sqrt{T}} + \sqrt{\epsb} + \sqrt{\epsap} \right)
        \\\label{e: opt.violation}
        -\frac{1}{K}\sum_{k=0}^{K-1} \bE\left[ g(\bJ^{\theta_k}_c)\right] & = \tilde{\cO}\left(\frac{1}{\sqrt{T}} + \sqrt{\epsb} + \sqrt{\epsap} \right).
    \end{align}
    %
    %
\end{theorem}

\begin{remark}[\textbf{\textit{Optimal Convergence rate}}]\label{rem: opt.converge}
    Equation~\eqref{e: opt.converge} shows that objective error decays to zero at the optimal rate of $\tilde{O}(1/\sqrt{T}),$ modulo the parametrization error $\epsb$ in the policy class, and the function approximation error $\epsap$ in critic estimation, where $T$ is the horizon length.
\end{remark}
\begin{remark}[\textbf{\textit{Optimal constraint violation}}]\label{rem: opt.violation}
    Equation~\eqref{e: opt.violation} captures the expected rate of constraint violation, and says that this rate also decays to zero at the optimal rate of $\tilde{O}(1/\sqrt{T}),$ modulo parameterization and critic estimation error.
\end{remark}
\begin{remark}[\textbf{\textit{Unknown mixing-time}}]\label{rem: opt.mixing.time}
    The stepsizes $\alpha, \beta, \gamma_\xi,$ and $\gamma_w,$ as well as the horizon lengths $\Tm, H,$ and $K,$ require no knowledge of the mixing-time of the underlying CMDP.
\end{remark}

\vspace{-.1in}
\section{Proof outline for Theorem \ref{thm: main result} }
\label{s: proof.outline}
\vspace{-.1in}

In this section, we provide the outline for the proof of Theorem~\ref{thm: main result}. Our proof involves the following key steps, whose details are given in the appendix:
\begin{enumerate}[leftmargin=0pt,labelindent=0pt,labelwidth=!,itemindent=6pt,align=left,itemsep=2pt,topsep=2pt,parsep=0pt]

    \item Using smoothness of $J^\theta_u$ and the parameter update rule~\ref{e: actor.update} to bound the Lagrange error in terms of the NPG estimation error in Lemma~\ref{lem: gen.framework}. 
    \item Exploit the convergence of linear stochastic recursion (Lemma~\eqref{lem: gen.linear}) and MLMC estimation  (Lemma~\eqref{app.lem: MLMC}), to derive NPG estimation error in terms of the critic and scalar gradient error.
    \item Obtain critic and scalar gradient bounds: (i) pose critic update as a general linear recursion to invoke Lemma~\ref{lem: gen.linear}; and (ii) invoke MLMC properties to bound scalar estimators.
    \item Bound the expected objective error and constraint violation rate by relating them to the Lagrange error and setting the appropriate outer and inner loop iterations $K$ and $H.$ 
\end{enumerate}
We now elaborate on the steps mentioned above.

\vspace{-.1in}
\paragraph{Step 1. General framework for primal-dual policy gradient:}

Borrowing from the proof strategies from~\citep{xu2026global}, we begin by establishing the convergence of the Lagrange function.
\begin{lemma}[\textbf{\textit{Lagrange error}}]\label{lem: gen.framework}
    Let the parameter pair $(\theta_k,\lambda_k)$ be updated by \eqref{e: actor.update}. Under assumptions~\ref{a: score.function}--\ref{a: FA.error.PG}, we have for every $K>0,$
    \begin{multline*}
        \frac{1}{K}\sum_{k=0}^{K-1}\bE\left[\cL(\piS, \lambda_k) - \cL(\theta_k,\lambda_k) \right] - \frac{1}{\alpha K}\bE_{s\sim d^{\piS}}\KL\left(\piS(\cdot| s)||\pi_{\theta_K}(\cdot|s)\right)
        \\
        \leq \sqrt{\epsb} +  \frac{G_1}{K}\sum_{k=0}^{K-1}\bE\|\bE_k[w_k] - \wS_k\| + \frac{G_2\alpha}{K}\sum_{k=0}^{K-1}\bE\|w_k-\wS_k\|^2 
        + \alpha\tmix\frac{4G_1G_2}{\mu^2_F}\left( 1+ \frac{2}{\delta}\right),
    \end{multline*}
    where $\KL(\cdot||\cdot)$ denotes the Kullback-Leibler divergence.
\end{lemma}
%
%

Before proceeding, we present the following lemma which bounds the convergence rate of a general linear recursion:
\begin{equation}
    x_{h+1} = x_h + \bbeta\big[\hat{q}_h - \hat{P}_h x_h \big],
\end{equation}
where $(\hat{q}_h, \hat{P}_h)_{h=0}^{H-1}$ are noisy estimates for $(q, P)\in \bR^d\times\bR^{d\times d},$ and the target is an $\xS$ such that $P\xS=q.$
\begin{lemma}[\textbf{{\citep[Lemma B.1]{xu2026global}}\textit{ Linear stochastic iteration}}]\label{lem: gen.linear}
    Suppose that $(\hat{q}_h,\hat{P}_h)$ satisfies:
    
    $\bE_h\|\hat{P}_h - P \|^2 \leq \sigma^2_P,\quad \|\bE_h[\hat{P}_h] - P \|^2 \leq \delta^2_P,\quad 0< \lambda_P\leq \|P\|\leq \Lambda_P, \quad \|q\|\leq \Lambda_q,\quad 8\delta_P<\lambda_P,$ 

    $\bE_h\|\hat{q}_h - q \|^2 \leq \sigma^2_q, \quad \|\bE_h[\hat{q}_h] - q \|^2 \leq \delta^2_q, \quad \|\bE[\hat{q}_h] - q \|^2 \leq \bar{\delta}^2_q, \quad\text{and } \quad \bbeta(24\sigma^2_P + 8\Lambda^2_P)\leq \lambda_P.$ 
    
    Then, after $H$ iterations, we have

    $\qquad\quad \bE\|x_H - \xS\|^2  \leq \|x_0 - \xS\|^2\ e^{-\bbeta\lambda_P H} + \cO\left( \bbeta R_0 + R_1 \right),$

    $\qquad\quad \|\bE[x_H] - \xS \|^2 \leq \|x_0 - \xS\|^2\ e^{-\bbeta\lambda_P H} + \cO\bigg( \frac{\delta^2_P}{\lambda^2_P}\Big(\|x_0 - \xS\|^2 + \bbeta R_0 + R_1 \Big) + \bar{R}_1 \bigg),$
    
    %
    where $R_0:= (\Lambda^2_q/\lambda^3_P)\sigma^2_P + (1/\lambda_P)\sigma^2_q,$ $R_1:= (\Lambda_q^2/\lambda_P^4)\delta^2_P + (1/\lambda^2_P)\delta^2_q,$ and $\bar{R}_1:= (\Lambda_q^2/\lambda_P^4)\delta^2_P + (1/\lambda^2_P)\bar{\delta}^2_q.$
\end{lemma}
In the subsequent steps, we apply this result to establish our critic and NPG errors by posing them as special cases of Lemma~\ref{lem: gen.linear}.

\paragraph{Step 2. NPG estimator analysis:}

Recall from~\eqref{e: NPG.update} that the NPG update rule is given as follows.
\[
    w^k_{h+1} = w^k_h + \gamma_w\left[ \hat{q}^{kh}_\MLMC - \hat{F}^{kh}_\MLMC w^k_h \right],
\]
where $\ \hat{q}^{kh}_{\MLMC}:= \sum_{m=1}^{M}\partial_mf(\hat{\bJ}^{\theta_k}_{r,\MLMC})\hat{b}^{kh}_{r_m,\MLMC} + \lambda_k\sum_{n =1}^{N}\partial_n g(\hat{\bJ}^{\theta_k}_{c,\MLMC})\hat{b}^{kh}_{c_n,\MLMC},$ 

with $\hat{F}^{kh}_\MLMC,$ $\{\hat{b}^{kh}_{u,\MLMC}\} _{u\in\{r_m,c_n\}},$ as defined in~\eqref{e: MLMC}, while $\partial_m f(\hbJ_{r,\MLMC})$ and $\partial_n g(\hbJ_c)$ are as defined in~\eqref{e: scalar.grad.MLMC}.

To invoke Lemma~\ref{lem: gen.linear}, we bound the bias and variance in $(\hat{F}^{kh}_\MLMC, \hat{q}^{kh}_\MLMC)$ as an estimator for $(F(\theta_k, \nabla_\theta\cL(\theta_k, \lambda_k))).$ Exploiting the properties of MLMC-based estimation, we can show using Lemma~\ref{app.lem: MLMC} that the following result holds.
\begin{lemma}\label{lem: bias.variance.NPG} 
    Under Assumption~\ref{a: score.function}, we have

        $\big\|\bE_{kh}\big[\hat{F}^{kh}_\MLMC\big] - F(\theta_k)\big\|^2 = \cO\big( \tmix\Tm^{-1} \big)
        \quad \text{and} \quad \bE_{kh}\big\|\hat{F}^{kh}_\MLMC - F(\theta_k) \big\|^2 =  \cO\big( \tmix\ln\Tm \big).$

        
        Additionally, under Assumptions~\ref{a: smoothness} and~\ref{a: feature.degenrate}, we have

        $\left\|\bE_{kh}\big[\hat{q}^{kh}_\MLMC\big] - \nabla_\theta\cL(\theta_k, \lambda_k) \right\|^2 
        = \tilde{\cO}\bigg( \delta^2_{fg} + \big(1+ \delta^2_{fg}\big)\left[ \frac{\tmix}{\Tm} + \|\xi_{u,k} - \xi^*_{u,k}\|^2  + \epsap \right]\bigg),$

        $\left\|\bE_k\big[\hat{q}^{kh}_\MLMC\big] - \nabla_\theta\cL(\theta_k, \lambda_k) \right\|^2 
        = \tilde{\cO}\bigg( \delta^2_{fg} + \big(1+ \delta^2_{fg}\big)\left[ \frac{\tmix}{\Tm} + \|\bE_k[\xi_{u,k}] - \xi^*_{u,k}\|^2  + \epsap \right]\bigg),$

        $\bE_{kh}\left\|\hat{q}^{kh}_\MLMC - \nabla_\theta\cL(\theta_k, \lambda_k) \right\|^2
        = \tilde{\cO}\bigg( \sigma^2_{fg} + (1+\sigma^2_{fg})\left[\tmix \|\xi_{u,k} - \xi^{*}_{u,k}\|^2 + \epsap \right]\bigg),$
        
        %
        where
        \begin{align*}
            \delta^2_{fg} & := \big|\bE_k\big[ \partial_m f(\hat{\bJ}^{\theta_k}_{r,\MLMC})\big] - \partial f_m(\bJ^{\theta_k}_r)\big|^2 + \big|\bE_k\big[\partial g_n(\hat{\bJ}^{\theta_k}_{c,\MLMC})\big] - \partial g_n(\bJ^{\theta_k}_c) \big|^2,
            \\
            \sigma^2_{fg} &:= \bE_k\big| \partial_m f(\hat{\bJ}^{\theta_k}_{r,\MLMC}) - \partial f_m(\bJ^{\theta_k}_r)\big|^2 + \bE_k\big|\partial g_n(\hat{\bJ}^{\theta_k}_{c,\MLMC}) - \partial g_n(\bJ^{\theta}_c) \big|^2.
        \end{align*}
    \end{lemma}
Combining the above bounds with Lemma~\ref{lem: gen.linear} yields the following convergence bounds for~\eqref{e: NPG.update}
\begin{theorem}[\textbf{\textit{NPG error}}]\label{thm: NPG.error}
     Choose $\gamma_w=\frac{2\ln T}{\mu_F H}$ and $\Tm = T \geq \frac{8G_1\tmix}{\mu_F}.$ Under Assumptions~\ref{a: smoothness}--\ref{a: Fisher.degenrate}, and~\ref{a: feature.degenrate} we have for $0\leq k<K,$
    \begin{align*}
        & \left\|\bE_k[w_{k}] - \wS_{k}\right\|^2 = \tilde{\cO}\bigg( \frac{1}{T^2} +\delta^2_{fg} + \big(1+ \delta^2_{fg}\big)\left[ \frac{\tmix}{T} + \|\bE_k[\xi_{u,k}] - \xi^*_{u,k}\|^2  + \epsap \right]\bigg),
        \\
        & \bE_k\|w_{k} - \wS_{k}\|^2 =  \tilde{\cO}\bigg( \frac{1}{T^2}+ \frac{\tmix}{H}+\bigg(\delta^2_{fg} + \frac{\sigma^2_{fg}}{H^2}\bigg) 
        \\
        & \quad\qquad\qquad\qquad\qquad + \bigg(2 + \delta^2_{fg} + \frac{\sigma^2_{fg}}{T}\bigg)\left[\tmix \|\xi_{u,k} - \xi^{*}_{u,k}\|^2 + \epsap \right] \bigg),
    \end{align*}
    where $\delta^2_{fg}$ and $\sigma^2_{fg}$ are as defined in Lemma~\ref{lem: bias.variance.NPG}.
\end{theorem}

\vspace{-.1in}
\paragraph{Step 3.(i) Critic estimator analysis:}

Recall from~\eqref{e: critic.joint.update} that the critic parameter $\xi^k_{u,h}$ is updated as
\[
    \xi^k_{u,h+1} = \xi^k_{u,h} + \gamma_\xi\left[\hat{q}^{kh}_{u,\MLMC} - \hat{P}^{kh}_{\MLMC}\ \xi^k_{u,h} \right],
\]
where $(\hat{q}^{kh}_{u, \MLMC}, \hat{P}^{kh}_\MLMC)_{h=0}^{H-1}$ are MLMC-based noisy estimates for $(q_u(\theta_k), P(\theta_k)).$ To employ Lemma~\ref{lem: gen.linear}, we first bound the bias and variance in these estimates.

Invoking the properties of MLMC estimation as stated in Lemma~\ref{app.lem: MLMC}, we prove the following bounds.

\begin{lemma}\label{lem: critic.bias.variance}
     Under Assumption~\ref{a: feature.degenrate}, the following holds for all $0\leq k<K,$ and $ 0\leq h\leq H-1:$ 

    $\bE_{kh}\|\hat{P}^{kh}_{\MLMC} - P(\theta_k)\|^2 = \cO(\tmix\ln\Tm);  \qquad \|\bE_{kh}[\hat{P}^{kh}_{\MLMC}] - P(\theta_k)\|^2 = \cO(\tmix\Tm^{-1});$

    $\bE_{kh}\|\hat{q}^{kh}_{u, \MLMC} - q_u(\theta_k)\|^2 = \cO(\tmix\ln\Tm); \quad \|\bE_{kh}[\hat{q}^{kh}_{u, \MLMC}] - q_u(\theta_k)\|^2 = \cO(\tmix\Tm^{-1}).$
\end{lemma}

Choosing $c_\gamma > \mu_\phi + \sqrt{\frac{1}{\mu_\phi^{2}}-1}$ guarantees, under Assumption~\ref{a: feature.degenrate},  that $P(\theta_k)-\frac{\mu_\phi}{2}I$ is positive-semidefinite; consequently, $P(\theta_k)$ is invertible. Moreover, $\xi^*_{u,k}:= P(\theta_k)^{-1}q_u(\theta_k)$ is equal to $[J^{\theta_k}_u, \zeta^{\theta_k}_u]^\top$. 
Thus, we can combine Lemma~\ref{lem: gen.linear} with the bounds from Lemma~\ref{lem: critic.bias.variance}, to obtain the following result.
\begin{theorem}[\textbf{\textit{Critic error}}]\label{thm: critic.bounds}
    Choose  $c_\gamma > \mu_\phi + \sqrt{\frac{1}{\mu_\phi^{2}}-1},$ $\gamma_\xi = \frac{2\ln T}{\mu_\phi H},$ and $\Tm = T.$ Then, under Assumption~\ref{a: feature.degenrate}, the following holds for sufficiently large $H:$

    $\quad \bE_k\big\|\xi^k_{u,H} - \xi^*_{u,k}\big\|^2 = \tilde{\cO}\left( \tmix T^{-1} +  \tmix H^{-1}\right) \quad \text{and} \quad \big\|\bE_k[\xi^k_{u,H}] - \xi^*_{u,k}\big\|^2 = \tilde{\cO}\left( \tmix^2T^{-1}\right).$
\end{theorem}

\paragraph{Step 3(ii). Scalarization estimator analysis:}

Using the definition of $\partial_m f(\hat{\bJ}^{\theta_k}_{r,T})$ and $\partial_n g(\hat{\bJ}^{\theta_k}_{c,T}),$ from~\eqref{e: scalar.grad.MLMC}, and leveraging the fact that trajectory length follows $\textnormal{Geom}(1/2),$ we can show the following.
\begin{theorem}\label{thm: bias.vari.scalar}
    The MLMC-based estimators for the scalar gradients satisfy:
    \begin{enumerate}[noitemsep, label=(\roman*)]
    \item $\bE_k\left[g(\hat{\bJ}^{\theta_k}_{c,\MLMC}) \right] = \bE_k\left[ g(\hat{\bJ}^{\theta_k}_{c,2^{\floor{\log_2(\Tm)}}}) \right];$
    \item $\bE_k\left[\partial_m f(\hat{\bJ}^{\theta_k}_{r,\MLMC}) \right] = \bE_k\left[\partial_m f(\hat{\bJ}^{\theta_k}_{r,2^{\floor{\log_2(\Tm)}}}) \right];$
    \item $\bE_k\left[\partial_n g(\hat{\bJ}^{\theta_k}_{c,\MLMC}) \right] = \bE_k\left[\partial_n g(\hat{\bJ}^{\theta_k}_{c,2^{\floor{\log_2(\Tm)}}}) \right].$
\end{enumerate}
    Consequently,  $\delta^2_{fg} = \cO\left((M^2+N^2)\tmix\Tm^{-1} \right) \quad \text{and} \quad  \sigma^2_{fg} = \cO\left( (M^2+N^2)\tmix\ln\Tm \right).$
\end{theorem}

\vspace{-.1in}
\paragraph{Step 4. Global convergence and constraint violation:} We choose $\Tm=T$ and $H=\Theta(\ln T).$ Then, by combining the scalar gradient, critic, and NPG bounds from Theorems~\ref{thm: bias.vari.scalar},~\ref{thm: critic.bounds}, and~\ref{thm: NPG.error}, and plugging them into Lemma~\ref{lem: gen.framework}, we can obtain the final Lagrange error as:
\begin{equation}\label{e: final.error.lagrange}
        \frac{1}{K}\sum_{k=0}^{K-1}\bE\left[ \cL(\piS, \lambda_k) - \cL(\theta_k, \lambda_k) \right]
        \leq  \tilde{\cO}\left(\sqrt{\epsb } +\sqrt{\epsap}+ \frac{\tmix^2}{\sqrt{T}} + \alpha + \frac{1}{\alpha K}\right).
\end{equation}
Further, using the dual update rule in~\eqref{e: actor.update}, Assumption~\ref{a: slater.feasibility}, and Theorem~\ref{thm: bias.vari.scalar}.(a), we can obtain
\begin{multline*}
    -\frac{1}{K}\sum_{k=0}^{K-1}\bE\left[\lambda_k\left(g(\bJ^{\piS}_c) - g(\bJ^{\theta_k}_c) \right)\right]
    \\
    \leq  \frac{2}{\delta K}\sum_{k=0}^{K-1}\bE\left| g(\bJ^{\theta_k}_c) - \bE_k[g(\hbJ_{c,\MLMC})\big]\right| + \frac{2}{\delta^2}\beta \overset{(a)}{=} \cO\left( \frac{\tmix}{\sqrt{T}} + \beta \right),
\end{multline*}
where $(a)$ follows from Theorem~\ref{thm: bias.vari.scalar} and taking $\Tm=T.$

We combine the above relation with~\eqref{e: final.error.lagrange}, and use the definition of $\cL$ to get
\begin{equation}\label{e: final.error.objective}
         f(\bJ^{\piS}_r)- \frac{1}{K}\sum_{k=0}^{K-1}
         \bE\big[f(\hbJ_r)\big]
        = \tilde{\cO}\bigg(\sqrt{\epsb } +\sqrt{\epsap}+ \frac{\tmix^2}{\sqrt{T}} + \alpha + \frac{1}{\alpha K}+\beta
         \bigg).
\end{equation}
Further, since $0\leq \lambda_k\leq \frac{2}{\delta}$ and $g(\bJ^{\piS}_c)\geq 0,$ we have from~\eqref{e: final.error.lagrange} that
\begin{multline}\label{e: total.sum.bound}
    f(\bJ^{\piS}_r) - \frac{1}{K}\sum_{k=0}^{K-1} \bE\big[f(\bJ^{\theta_k}_r)\big] + \frac{2}{\delta K}\sum_{k=0}^{K-1}\bE\big[-g(\bJ^{\theta_k}_c) \big] 
    \\
    \leq \frac{1}{K}\sum_{k=0}^{K-1}\bE\left[g(\bJ^{\theta_k}_c)\left(\lambda_k - \frac{2}{\delta} \right) \right] + \frac{1}{K}\sum_{k=0}^{K-1}\bE\left[\cL(\piS, \lambda_k) - \cL(\theta_k,\lambda_k)\right]
    \\
    = \tilde{\cO}\bigg(\sqrt{\epsb } +\sqrt{\epsap}+ \frac{\tmix^2}{\sqrt{T}} + \alpha + \frac{1}{\alpha K}+\beta + \frac{1}{\beta K}
         \bigg),
\end{multline}

where the last equality is obtained by exploiting the dual update rule and the MLMC bounds on cost estimate $g_k.$ Finally, combining this with~\citep[Lemma~G.6]{xu2026global} finishes the proof as we set $\alpha,\beta,$ and $K$ as prescribed.  \hfill$\qed$

\vspace{-.1in}
\section{Conclusions}
\label{s: conclusion}

We studied constrained multi-objective reinforcement learning in the infinite-horizon average-reward setting under concave scalarization and parametrized policies. We proposed an MLMC-based primal-dual Natural Actor-Critic algorithm that jointly controls biases arising from scalarization, actor-critic estimation, and dual updates, while avoiding any dependence on mixing-time knowledge. We established optimal $\tilde{\cO}(1/\sqrt{T})$ convergence and constraint violation rates, providing the first finite-time guarantees for this setting. A key avenue for future work is to preserve the established guarantees under more general assumptions.
\newpage



\bibliographystyle{plain}

\bibliography{references}

\newpage 


\appendix


\section{Pseudocode}
\label{appendix: algorithm}

\begin{algorithm}[h]
\caption{Multi-Objective Primal-Dual NAC (MO-PDNAC) \label{alg: MOPDNAC}}
\SetKwInOut{Input}{Input}
\SetKw{Initialize}{Initialize}
\SetKwFor{For}{for}{}{end}
\SetKwBlock{NPG}{NPG estimator}{}
\SetKwBlock{Critic}{Critic estimator}{}
\SetKwBlock{scalar}{Scalar estimators}{}
\Input{Parameter class $\Theta,$ feature matrix $\Phi,$ primal stepsize $\alpha,$ dual stepsize $\beta,$ NPG stepsize $\gamma_w,$ critic stepsize $\gamma_\xi,$ initial parameters $\theta_0\in \Theta,$ $w_0=0,$ $\xi_0=0,$ outer loop size $K,$ inner loop size $H,$ trajectory length $\Tm.$}




\For{Outer loop iteration $k= 0,1, \ldots, K-1:$}
   { 
        Set $w^k_0 = \xi^k_{u,0} = 0,$ $u\in \{r_m,c_n: m\in [M], n\in [N]\}.$
        
        \scalar 
        {
            Sample $s^k_0\sim\rho,$ and $Q_k\sim \textnormal{Geom}(1/2)$ and set $\ell_k:= 1+ (2^{Q_k}-1)\ones_{\{2^{Q_k}<\Tm\}}.$ 

            \For{$t=0,1,\ldots, \ell_k-1:$}
            {
                Sample action $a^k_t\sim\pi_{\theta_k}(\cdot|s^k_t),$ observe state $s^k_{t+1}\sim \cP(\cdot|s^k_t,a^k_t)$ and $u(s^k_t,a^k_t),$ $u\in\{r_m,c_n: m\in[M], n\in[N]\}.$
            }
            
            Compute $g(\hat{\bJ}^{\theta_k}_{c,\MLMC}),$ $\partial_m f(\hat{\bJ}^{\theta_k}_{r,\MLMC}), \partial_n g(\hat{\bJ}^{\theta_k}_{c,\MLMC})$ via~\eqref{e: scalar.grad.MLMC} and set $g_k=g(\hat{\bJ}^{\theta_k}_{c,\MLMC})$.
            
        }

        \Critic
        {
            \For{Inner loop iterations $h=0,1,\dots,H-1:$}
            {
                Sample $s^{kh}_0\sim\rho,$ and $Q_{kh}\sim \textnormal{Geom}(1/2)$ and set $\ell_{kh}:= 1+ (2^{Q_{kh}}-1)\ones_{\{2^{Q_{kh}}<\Tm\}}.$

                \For{$t=0,1,\ldots, \ell_{kh}-1:$}
                {
                    Sample action $a^{kh}_t\sim\pi_{\theta_k}(\cdot|s^{kh}_t),$ observe state $s^{kh}_{t+1}\sim \cP(\cdot|s^{kh}_t,a^{kh}_t)$ 

                    Compute $\hat{q}^k_u(s^{kh}_t, a^{kh}_t), \hat{P}(s^{kh}_t, a^{kh}_t, s^{kh}_{t+1})$ via~\eqref{e: critic.naive.single.step}
                }

                Compute $(\hat{q}^{kh}_{u,\MLMC}, \hat{P}^{kh}_{\MLMC})$ using~\eqref{e: MLMC}, and update critic parameter $\xi^k_{u,h}$ via~\eqref{e: critic.joint.update}.
                
            }

            Set $\xi_{u,k} = \xi^k_{u,H}.$
        }
        \NPG
        {
            \For{Inner loop iterations $h=0,1,\dots,H-1:$}
            {
                Sample $s^{kh}_0\sim\rho,$ and $Q_{kh}\sim \textnormal{Geom}(1/2)$ and set $\ell_{kh}:= 1+(2^{Q_{kh}}-1) \ones_{\{2^{Q_{kh}}<\Tm\}}.$

                \For{$t=0,1,\ldots, \ell_{kh}-1:$}
                {
                    Sample action $a^{kh}_t\sim\pi_{\theta_k}(\cdot|s^{kh}_t),$ observe state $s^{kh}_{t+1}\sim \cP(\cdot|s^{kh}_t,a^{kh}_t)$ 

                    Compute $\hat{b}^k_u(s^{kh}_t, a^{kh}_t), \hat{F}^{kh}(s^{kh}_t, a^{kh}_t, s^{kh}_{t+1})$ via~\eqref{e: NPG.naive.single.step}
                }

                Compute $(\hat{b}^{kh}_{u,\MLMC}, \hat{F}^{kh}_{\MLMC})$ using~\eqref{e: MLMC}, and update NPG parameter $w^k_h$ via~\eqref{e: NPG.update}.
                
            }

            Set $w_k = w^k_H.$
        }
        
    Update $(\theta_k, \lambda_k)$ using~\eqref{e: actor.update}.
    
   }
\end{algorithm}


\newpage

\section{Proof of Lemma~\ref{lem: gen.framework}: Analysis of general primal-dual policy gradient}
\label{appendix: gen.framework}
%
    %
    For $0\leq k<K,$ we have
    \begin{align*}
        &\bE_{s\sim d^{\piS}}\left[ \KL\left(\piS(\cdot| s)||\pi_{\theta_k}(\cdot|s)\right) - \KL\left(\piS(\cdot|s)|| \pi_{\theta_{k+1}}(\cdot|s)\right) \right] = \bE_{(s,a)\sim \nu^{\piS}}\left[\ln\frac{\pi_{\theta_{k+1}}(a|s)}{\pi_{\theta_k}(a|s)} \right]
        \\
        & \overset{(a)}{\geq} \bE_{(s,a)\sim \nu^{\piS}}\left\langle\nabla_\theta \ln\pi_{\theta_k}(a|s), \theta_{k+1} -\theta_k \right\rangle - \frac{G_2}{2}\|\theta_{k+1} - \theta_k\|^2
        \\
        & \overset{(b)}{=} \alpha\bE_{(s,a)\sim \nu^{\piS}}\left\langle\nabla_\theta \ln\pi_{\theta_k}(a|s), w_k \right\rangle  - \frac{G_2\alpha^2}{2}\|w_k\|^2 
        \\
        & \geq \alpha\bE_{(s,a)\sim \nu^{\piS}}\left\langle\nabla_\theta \ln\pi_{\theta_k}(a|s), \wS_k \right\rangle + \alpha\bE_{(s,a)\sim \nu^{\piS}}\left\langle\nabla_\theta \ln\pi_{\theta_k}(a|s), w_k - \wS_k \right\rangle 
        - \frac{G_2\alpha^2}{2}\|w_k\|^2  
        \\
        & \overset{(c)}{=} \alpha\left[ \cL(\piS,\lambda_k) - \cL(\theta_k,\lambda_k) \right]  + \alpha\bE_{(s,a)\sim \nu^{\piS}}\left\langle\nabla_\theta \ln\pi_{\theta_k}(a|s), \wS_k \right\rangle 
         - \frac{G_2\alpha^2}{2}\|w_k\|^2   
        \\
        & \quad-\alpha\left[ f(\bJ^{\piS}_r) - f(\hbJ_r) \right] - \alpha\lambda_k\left[ g(\bJ^{\piS}_c) - g(\hbJ_c) \right] 
        + \alpha\bE_{(s,a)\sim \nu^{\piS}}\left\langle\nabla_\theta \ln\pi_{\theta_k}(a|s), w_k - \wS_k \right\rangle 
        \\
        & \overset{(d)}{=} \alpha\left[ \cL(\piS,\lambda_k) - \cL(\theta_k,\lambda_k) \right] 
        + \alpha\bE_{(s,a)\sim \nu^{\piS}}\left\langle\nabla_\theta \ln\pi_{\theta_k}(a|s), w_k - \wS_k \right\rangle         - \frac{G_2\alpha^2}{2}\|w_k\|^2   
        \\
        & + \alpha\bE_{(s,a)\sim \nu^{\piS}}\left[\left\langle\nabla_\theta \ln\pi_{\theta_k}(a|s), \wS_k \right\rangle - \sum_{m=1}^M \partial_m f(\hbJ_r) A_{r_m}^{\theta_k}(s,a) - \lambda_k\sum_{n=1}^N \partial_n g(\hbJ_c) A_{c_n}^{\theta_k}(s,a) \right]
        \\
        & \overset{(e)}{\geq} \alpha\left[ \cL(\piS,\lambda_k) - \cL(\theta_k,\lambda_k) \right] 
        + \alpha\bE_{(s,a)\sim \nu^{\piS}}\left\langle\nabla_\theta \ln\pi_{\theta_k}(a|s), w_k - \wS_k \right\rangle         - \frac{G_2\alpha^2}{2}\|w_k\|^2 
        \\
        &  -\alpha\sqrt{\bE_{(s,a)\sim \nu^{\piS}}\left[\left\langle\nabla_\theta \ln\pi_{\theta_k}(a|s), \wS_k \right\rangle - \sum_{m=1}^M \partial_m f(\hbJ_r) A_{r_m}^{\theta_k}(s,a) - \lambda_k\sum_{n=1}^N \partial_n g(\hbJ_c) A_{c_n}^{\theta_k}(s,a) \right]^2}
        \\
        & \overset{(f)}{\geq} \alpha\left[ \cL(\piS,\lambda_k) - \cL(\theta_k,\lambda_k) \right] 
        + \alpha\bE_{(s,a)\sim \nu^{\piS}}\left\langle\nabla_\theta \ln\pi_{\theta_k}(a|s), w_k - \wS_k \right\rangle         - \frac{G_2\alpha^2}{2}\|w_k\|^2 - \alpha\sqrt{\epsb},
    \end{align*}
    where $(a)$ follows from Assumption~\ref{a: score.function}, $(b)$ follows from~\eqref{e: actor.update}, $(c)$ follows from~\eqref{e: lagrange.function}, $(d)$ follows from Lemma~\ref{app.lem: performance.diff}, $(e)$ follows from Jensen's inequality since for $x>0,$ the function $x\mapsto \sqrt{x}$ is concave, and $(f)$ uses Assumption~\ref{a: FA.error.PG}.

    Now, rearranging terms and taking expectation gives
    \begin{align*}
        & \bE\left[\cL(\piS, \lambda_k) - \cL(\theta_k,\lambda_k) \right] - \frac{1}{\alpha}\bE_{s\sim d^{\piS}}\left[ \KL\left(\piS(\cdot| s)||\pi_{\theta_k}(\cdot|s)\right) - \KL\left(\piS(\cdot|s)|| \pi_{\theta_{k+1}}(\cdot|s)\right) \right]
        \\
        & \quad \leq \sqrt{\epsb} - \bE\bE_{(s,a)\sim \nu^{\piS}}\left\langle \nabla_\theta \ln \pi_{\theta_k}(a|s), w_k - \wS_k\right\rangle + \frac{G_2\alpha}{2}\bE\|w_k\|^2
        \\
        & \quad \overset{(a)}{\leq} \sqrt{\epsb} -\bE\left\langle \bE_{(s,a)\sim \nu^{\piS}}\nabla_\theta \ln \pi_{\theta_k}(a|s), \bE_k[w_k] - \wS_k\right\rangle + \frac{G_2\alpha}{2}\bE\|w_k\|^2 
        \\
        & \quad \overset{(b)}{\leq} \sqrt{\epsb} + \bE\left[\|\bE_{(s,a)\sim \nu^{\piS}}\pi_{\theta_k}(a|s)\|\|\bE_k[w_k] - \wS_k\|\right] + G_2\alpha\bE\|w_k-\wS_k\|^2 + G_2\alpha\bE\|\wS_k\|^2
        \\
        & \quad \overset{(c)}{\leq} \sqrt{\epsb} +  G_1\bE\|\bE_k[w_k] - \wS_k\| + G_2\alpha\bE\|w_k-\wS_k\|^2 + \frac{G_2\alpha}{\mu^2_F}\bE\|\nabla_\theta\cL(\theta_k,\lambda_k)\|^2 ,
    \end{align*}
    where $(a)$ follows since $\theta_k$ is $\cF_k$-measurable and $w_k$ is independent of $\nu^{\piS},$ $(b)$ uses the Cauchy-Schwarz inequality and the inequality $(x+y)^2\leq 2x^2+2y^2,$ and $(c)$ follows from Assumptions~\ref{a: score.function} and~\ref{a: Fisher.degenrate}. Lastly, to obtain $(d)$ uses the definition of $\wS_k$ and Assumption~\ref{a: Fisher.degenrate}.

    Summing both sides over $0\leq k <K,$ and dividing by $K$ and using the fact that the KL-divergence is non-negative, we get
    \begin{align}\label{e: lagrange.initial}
        & \frac{1}{K}\sum_{k=0}^{K-1}\bE\left[\cL(\piS, \lambda_k) - \cL(\theta_k,\lambda_k) \right] - \frac{1}{\alpha K}\bE_{s\sim d^{\piS}}\KL\left(\piS(\cdot| s)||\pi_{\theta_K}(\cdot|s)\right) \leq \sqrt{\epsb} 
        \nonumber\\
        &   + \frac{G_1}{K}\sum_{k=0}^{K-1}\bE\|\bE_k[w_k] - \wS_k\| + \frac{G_2\alpha}{K}\sum_{k=0}^{K-1}\bE\|w_k-\wS_k\|^2 + \frac{G_2\alpha}{\mu^2_F K}\sum_{k=0}^{K-1}\bE\|\nabla_\theta\cL(\theta_k,\lambda_k)\|^2.
    \end{align}
    We can further bound the last term as follows. Note that $\|\nabla_\theta J^\theta_u\|\leq 4\tmix G_1$ for all $u\in\{r,c\}.$ Thus, $\|\nabla_\theta\cL(\theta_k,\lambda_k)\|\leq 4\tmix(1+ 2/\delta)G_1.$ This, gives us
    \begin{align}
        & \frac{1}{K}\sum_{k=0}^{K-1}\bE\left[\cL(\piS, \lambda_k) - \cL(\theta_k,\lambda_k) \right] - \frac{1}{\alpha K}\bE_{s\sim d^{\piS}}\KL\left(\piS(\cdot| s)||\pi_{\theta_K}(\cdot|s)\right) \leq \sqrt{\epsb} 
        \nonumber\\
        &   + \frac{G_1}{K}\sum_{k=0}^{K-1}\bE\|\bE_k[w_k] - \wS_k\| + \frac{G_2\alpha}{K}\sum_{k=0}^{K-1}\bE\|w_k-\wS_k\|^2 + \alpha\tmix\frac{4G_1G_2}{\mu^2_F}\left( 1+ \frac{2}{\delta}\right).
    \end{align}
    This completes the proof of Lemma~\ref{lem: gen.framework}. \hfill \qed
    %

\section{Proof of Lemma~\ref{lem: bias.variance.NPG}: Analysis of NPG estimation subroutine}
\label{appendix: NPG.analysis}

%
Note that, we can write
\begin{align*}
    &\hat{q}^{kh}_\MLMC- \nabla_\theta\cL(\theta_k, \lambda_k) 
    \\
    & = \sum_{m=1}^M \left(\partial_m f(\hat{\bJ}^{\theta_k}_{r,\MLMC}) - \partial_m f(\hbJ_r)\right)\nabla_\theta J^{\theta_k}_{r_m} + \sum_{m=1}^M\partial_m f(\hbJ_r)\left(\hat{b}^{kh}_{r_m,\MLMC} - \nabla_\theta J^{\theta_k}_{r_m} \right)
    \\
    & + \sum_{m=1}^M \left(\partial_m f(\hat{\bJ}^{\theta_k}_{r,\MLMC}) - \partial_m f(\hbJ_r)\right)\left(  \hat{b}^{kh}_{r_m,\MLMC} - \nabla_\theta J^{\theta_k}_{r_m} \right)
    \\
    & + \lambda_k\sum_{n=1}^N \left(\partial_n g(\hat{\bJ}^{\theta_k}_{c,\MLMC}) - \partial_n g(\hbJ_c)\right)\nabla_\theta J^{\theta_k}_{c_n} + \sum_{n=1}^N\partial_n g(\hbJ_c)\left(\hat{b}^{kh}_{c_n,\MLMC} - \nabla_\theta J^{\theta_k}_{c_n} \right)
    \\
    & + \lambda_k\sum_{n=1}^N \left(\partial_n g(\hat{\bJ}^{\theta_k}_{c,\MLMC}) - \partial_n g(\hbJ_c)\right)\left(  \hat{b}^{kh}_{c_n,\MLMC} - \nabla_\theta J^{\theta_k}_{c_n} \right).
\end{align*}
Hence, taking conditional expectation followed by $\|\cdot\|$ gives
\begin{align}\label{e: bias.PG}
    &\left\|\bE_{kh}\big[\hat{q}^{kh}_\MLMC\big] - \nabla_\theta\cL(\theta_k, \lambda_k) \right\|
    \overset{(a)}{\leq} \sum_{m=1}^M \left\|\left(\bE_{kh}\big[\partial_m f(\hat{\bJ}^{\theta_k}_{r,\MLMC})\big] - \partial_m f(\bJ^{\theta_k}_r)\right)\nabla_\theta J^{\theta_k}_{r_m}\right\| 
    \nonumber \\
    & \quad + \sum_{m=1}^M \left\| \partial_m f(\bJ^{\theta_k}_r)\left(\bE_{kh}\big[\hat{b}^{kh}_{r_m,\MLMC}\big] - \nabla_\theta J^{\theta_k}_{r_m} \right)\right\|
    \nonumber\\
    & \quad + \sum_{m=1}^M \left\|\bE_{kh}\left[\left(\partial_m f(\hat{\bJ}^{\theta_k}_{r,\MLMC}) - \partial_m f(\bJ^{\theta_k}_r)\right)\left(  \hat{b}^{kh}_{r_m,\MLMC} - \nabla_\theta J^{\theta_k}_{r_m} \right)\right]\right\|
    \nonumber\\
    & \quad + \frac{2}{\delta}\sum_{n=1}^N \left\|\left(\bE_{kh}\big[\partial_n g(\hat{\bJ}^{\theta_k}_{c,\MLMC})\big] - \partial_n g(\bJ^{\theta_k}_c)\right)\nabla_\theta J^{\theta_k}_{c_n}\right\| 
    \nonumber\\
    & \quad + \frac{2}{\delta}\sum_{n=1}^N \left\| \partial_n g(\bJ^{\theta_k}_c)\left(\bE_{kh}\big[\hat{b}^{kh}_{c_n,\MLMC}\big] - \nabla_\theta J^{\theta_k}_{c_n} \right)\right\|
    \nonumber\\
    & \quad + \frac{2}{\delta}\sum_{n=1}^N \left\| \bE_{kh}\left[\left(\partial_n g(\hat{\bJ}^{\theta_k}_{c,\MLMC}) - \partial_n g(\bJ^{\theta_k}_c)\right)\left(  \hat{b}^{kh}_{c_n,\MLMC} - \nabla_\theta J^{\theta_k}_{c_n} \right)\right]\right\|
    \nonumber\\
    & \quad \overset{(b)}{\leq} 4G_1\tmix\sum_{m=1}^M \left|\bE_{kh}\big[\partial_m f(\hat{\bJ}^{\theta_k}_{r,\MLMC})\big] - \partial_m f(\bJ^{\theta_k}_r)\right| + C_f\sum_{m=1}^M \left\| \bE_{kh}\big[\hat{b}^{kh}_{r_m,\MLMC}\big] - \nabla_\theta J^{\theta_k}_{r_m} \right\|
    \nonumber\\
    & \quad + \sum_{m=1}^M \left|\bE_{kh}\big[\partial_m f(\hat{\bJ}^{\theta_k}_{r,\MLMC})\big] - \partial_m f(\bJ^{\theta_k}_r)\right|\left\|
    \bE_{kh}\big[\hat{b}^{kh}_{r_m,\MLMC}\big] - \nabla_\theta J^{\theta_k}_{r_m} \right\|
    \nonumber\\
    & \quad + \frac{8G_1\tmix}{\delta}\sum_{n=1}^N \left|\bE_{kh}\big[\partial_n g(\hat{\bJ}^{\theta_k}_{c,\MLMC})\big] - \partial_n g(\bJ^{\theta_k}_c)\right| + \frac{2C_g}{\delta}\sum_{n=1}^N \left\| \bE_{kh}\big[\hat{b}^{kh}_{c_n,\MLMC}\big] - \nabla_\theta J^{\theta_k}_{c_n} \right\|
    \nonumber\\
    & \quad + \frac{2}{\delta}\sum_{n=1}^N \left|\bE_{kh}\big[\partial_n g(\hat{\bJ}^{\theta_k}_{c,\MLMC})\big] - \partial_n g(\bJ^{\theta_k}_c)\right|\left\|  \bE_{kh}\big[\hat{b}^{kh}_{c_n,\MLMC}\big] - \nabla_\theta J^{\theta_k}_{c_n} \right\|,
\end{align}
where $(a)$ follows since $\lambda_k$ is $\cF_{kh}$-measurable and $|\lambda_k|\leq \frac{2}{\delta},$ and to obtain $(b),$ we use the following facts: first, $|r_m|,|c_n|\leq 1$ and Assumption~\ref{a: score.function}, together imply $\|\nabla_\theta J^{\theta_k}_u\| \leq \bE_{\nu^{\theta_k}} \|A^{\theta_k}_u(s,a)\nabla_\theta\ln\pi_{\theta_k}(a|s)\| \leq G_1\bE_{\nu^{\theta_k}}[A^{\theta_k}_u(s,a)] \leq 4G_1\tmix;$ second, that  given $\cF_{kh},$ $\hat{b}^{kh}_{r_m,\MLMC}\perp \partial_m f(\hat{\bJ}^{\theta_k}_{r,\MLMC}) $ and $\hat{b}^{kh}_{c_n,\MLMC}\perp \partial_n g(\hat{\bJ}^{\theta_k}_{c,\MLMC});$ and lastly, Assumption~\ref{a: smoothness} implies there exists $C_f,C_g>0$ such that $\partial_m f$ and $\partial_n g$ are uniformly bounded by $C_f$ and $C_g,$ respectively in the relevant compact domains.

The analogous bounds hold for condtitional expectation $\bE_k.$

Similarly, we can show
\begin{align}\label{e: var.PG}
    & \bE_{kh}\left\|\hat{q}^{kh}_{u,\MLMC} - \nabla_\theta\cL(\theta_k, \lambda_k)\right\|^2
    \nonumber\\
    &\leq 96MG^2_1\tmix^2\sum_{m=1}^{M}\bE_{kh}\left\|\partial_m f(\hat{\bJ}^{\theta_k}_{r,\MLMC}) - \partial_m f(\bJ^{\theta_k}_r)\right\|^2 + 6MC^2_f\sum_{m=1}^M \bE_{kh}\left\| \hat{b}^{kh}_{r_m,\MLMC} - \nabla_\theta J^{\theta_k}_{r_m} \right\|^2
    \nonumber\\
    & + 6M\sum_{m=1}^M \bE_{kh}\left\|\partial_m f(\hat{\bJ}^{\theta_k}_{r,\MLMC}) - \partial_m f(\bJ^{\theta_k}_r)\right\|^2 \bE_{kh}\left\|
    \hat{b}^{kh}_{r_m,\MLMC} - \nabla_\theta J^{\theta_k}_{r_m} \right\|^2
    \nonumber\\
    & + \frac{384NG^2_1\tmix^2}{\delta^2}\sum_{n=1}^N \bE_{kh}\left\|\partial_n g(\hat{\bJ}^{\theta_k}_{c,\MLMC}) - \partial_n g(\bJ^{\theta_k}_c)\right\|^2 + \frac{24NC^2_g}{\delta^2}\sum_{n=1}^N \bE_{kh}\left\| \hat{b}^{kh}_{c_n,\MLMC} - \nabla_\theta J^{\theta_k}_{c_n} \right\|^2
    \nonumber\\
    & + \frac{8N}{\delta^2}\sum_{n=1}^N \bE_{kh}\left\|\partial_n g(\hat{\bJ}^{\theta_k}_{c,\MLMC}) - \partial_n g(\bJ^{\theta_k}_c)\right\|^2 \bE_{kh}\left\|  \hat{b}^{kh}_{c_n,\MLMC} - \nabla_\theta J^{\theta_k}_{c_n} \right\|^2,
\end{align}
where we again use the fact that $\hat{b}^{kh}_{r_m,\MLMC}\perp \partial_m f(\hat{\bJ}^{\theta_k}_{r,\MLMC}) $ and $\hat{b}^{kh}_{c_n,\MLMC}\perp \partial_n g(\hat{\bJ}^{\theta_k}_{c,\MLMC}),$ conditioned on $\cF_{kh}.$

Now, we use Lemma~\ref{app.lem: MLMC} to bound the bias and variance in $\hat{b}^{kh}_{u,\MLMC}$ as an estimator of $\nabla_\theta J^{\theta_k}_u.$ To do so, recall from~\eqref{e: NPG.naive.single.step} that the corresponding naive  estimator is $\frac{1}{T}\sum_{t=0}^{T-1}\hat{b}_{u}^k(s^{kh}_t,a^{kh}_t,s^{kh}_{t+1}),$ where
\[
    \hat{b}^{kh}_u(s,a,s') := \left( u(s,a) - \eta_{u,k} + [\phi(s') - \phi(s)]^\top\zeta_{u,k} \right)\cdot\nabla_\theta\ln \pi_{\theta_k}(a|s).
\]

Hence, the bias in the naive estimator is
\begin{align*}
    & \left\|\bE_{(s,a)\sim\nu^{\theta_k}, s'\sim \cP(\cdot|s,a)}[\hat{b}^{kh}_u(s,a,s')] - \nabla_\theta J^{\theta_k}_{u} \right\|
    \\
    & = \left\| \bE_{(s,a)\sim\nu^{\theta_k}}\left[ \left( u(s,a) - \eta_{u,k} + [\bE_{s'\sim\cP(\cdot|,s,a)}[\phi(s')] - \phi(s)]^\top\zeta_{u,k} \right)\cdot\nabla_\theta\ln \pi_{\theta_k}(a|s)  \right] - \nabla_\theta J^{\theta_k}_{u}\right\|
    \\
    & \leq \left\| \bE_{(s,a)\sim\nu^{\theta_k}}\left[ \left( u(s,a) - J^{\theta_k}_u + \bE_{s'\sim\cP(\cdot|s,a)}\big[ V^{\theta_k}_u(s')\big] - V^{\theta_k}_u(s) \right)\cdot\nabla_\theta\ln \pi_{\theta_k}(a|s)  \right] - \nabla_\theta J^{\theta_k}_{u}\right\|
    \\
    & + \left\|\bE_{(s,a)\sim\nu^{\theta_k}}\left[ \left(J^{\theta_k}_u - \eta_{u,k}\right)\nabla_\theta \ln\pi_{\theta_k}(a|s) \right]\right\| 
    \\
    & + 
    \left\|\bE_{(s,a)\sim\nu^{\theta_k}}\left[ \left(V^{\theta_k}_u(s) - \phi(s)^\top\zeta_{u,k}\right)\nabla_\theta \ln\pi_{\theta_k}(a|s) \right] \right\| 
    \\
    & + \left\|\bE_{(s,a)\sim \nu^{\theta_k}, s'\sim \cP(\cdot|s,a)}\left[ \left(V^{\theta_k}_u(s') - \phi(s')^\top\zeta_{u,k}\right)\nabla_\theta \ln\pi_{\theta_k}(a|s) \right]  \right\| 
    \\
    & \overset{(a)}{=} \left\| \bE_{(s,a)\sim\nu^{\theta_k}}\left[A^{\theta_k}_u(s,a)\ \nabla_\theta\ln \pi_{\theta_k}(a|s)  \right] - \nabla_\theta J^{\theta_k}_{u}\right\|
    + G_1\left\| J^{\theta_k}_u - \eta_{u,k}\right\| 
    \\
    & + 2G_1\left\|\bE_{(s,a)\sim\nu^{\theta_k}}\left[ V^{\theta_k}_u(s) - \phi(s)^\top\zeta_{u,k} \right] \right\| 
    \\
    & \overset{(b)}{\leq} G_1\left\| J^{\theta_k}_u - \eta_{u,k}\right\|  + 2G_1\left\|\bE_{s\sim d^{\theta_k}}\left[\phi(s)^\top\left(\zeta_{u,k} - \zeta^*_{u,k} \right) \right] \right\|
    \\
    & + 2G_1\left\|\bE_{(s,a)\sim\nu^{\theta_k}}\left[ V^{\theta_k}_u(s) - \phi(s)^\top\zeta^*_{u,k} \right] \right\| 
    \\
    & \overset{(c)}{\leq}  G_1\left\|J^{\theta_k}_u - \eta_{u,k} \right\| + 2G_1\left\| \zeta_{u,k} - \zeta^{\theta_k}_{u,k}\right\| + 2G_1\sqrt{\bE_{(s,a)\sim\nu^{\theta_k}}\left\|\left[ V^{\theta_k}_u(s) - \phi(s)^\top\zeta^*_{u,k} \right] \right\|^2 } 
    \\
    & \overset{(d)}{=} \cO\left( G_1\|\xi_{u,k} - \xi^*_{u,k}\| + \sqrt{\epsap}\right),
\end{align*}
where $(a)$ follows from Assumption~\ref{a: score.function} the definition of $A^{\theta_k}_u$ in \eqref{e: advantage.func}, $(b)$ uses the definition of the policy gradient in~\eqref{e: policy.gradient}, $(c)$ uses Jensen's inequality and $\|\phi(s)\|\leq 1$ from Assumption~\ref{a: score.function}, and $(d)$ follows from Assumption~\ref{a: FA.error.value}. 

Likewise, similar arguments show
\begin{align*}
    \left\|\bE_{(s,a)\sim\nu^{\theta_k}, s'\sim \cP(\cdot|s,a)} \bE_k[\hat{b}^{kh}_u(s,a,s')] - \nabla_\theta J^{\theta_k}_{u} \right\|
    \leq  \cO\left( G_1\|\bE_k[\xi_{u,k}] - \xi^*_{u,k}\| + \sqrt{\epsap}\right),
\end{align*}

Further, by Assumptions~\ref{a: score.function}, the naive estimator's variance satisfies
\begin{align*}
    &  \left\|\bE_{(s,a)\sim\nu^{\theta_k}, s'\sim \cP(\cdot|s,a)}\big[\hat{b}^{kh}_u(s,a,s')\big] - \hat{b}^{kh}_u(s^{kh}_t, a^{kh}_t, s^{kh}_{t+1})\right\|^2
    \\
    &\leq 4G_1^2\left( 2 + 2\|\xi_{u,k}\|^2 \right) = \cO(G^2_1\|\xi_{u,k}\|^2) \overset{(a)}{=} \cO(G^2_1\|\xi_{u,k} - \xi^*_{u,k}\|^2) + \cO(G^2_1c^2_\gamma \mu^{-2}_\phi),
\end{align*}
where $(a)$ uses the fact that $\|\xi^*_{u,k}\| = \|P(\theta_k)^{-1}q_u(\theta_k)\|,$ $\|q_u\|=1+c_\gamma$ and $\mu_\phi\leq \|P(\theta)\|$ from Assumption~\ref{a: feature.degenrate}.

Hence, using Lemma~\ref{app.lem: MLMC}, we conclude
\begin{align*}
    \left\|\bE_{kh}\big[\hat{b}^{kh}_{u,\MLMC}\big] - \nabla_\theta J^{\theta_k}_u\right\|^2 
    & = \cO\left( \sigma^2_b\tmix\Tm^{-1} + \delta^2_b \right) 
    \\
    & = \cO\left( \frac{G^2_1c^2_\gamma\tmix}{\mu^2_\phi\Tm} + G^2_1\|\xi_{u,k} - \xi^*_{u,k}\|^2  + \epsap\right),
\end{align*}
\begin{align*}
    \left\|\bE_k\big[\hat{b}^{kh}_{u,\MLMC}\big] - \nabla_\theta J^{\theta_k}_u\right\|^2 
    & = \cO\left( \sigma^2_b\tmix\Tm^{-1} + \delta^2_b \right) 
    \\
    & = \cO\left( \frac{G^2_1c^2_\gamma\tmix}{\mu^2_\phi\Tm} + G^2_1\|\bE_k[\xi_{u,k}] - \xi^*_{u,k}\|^2  + \epsap\right),
\end{align*}
and
\begin{align*}
    \bE_{kh}\left\|\hat{b}^{kh}_{u,\MLMC} - \nabla_\theta J^{\theta_k}_u\right\|^2
    & = \cO\left( \sigma^2_b\tmix\ln\Tm + \delta^2_b\right) 
    \\
    & = \cO\left( \frac{G^2_1c^2_\gamma}{\mu^2_\phi}\tmix\ln\Tm  + G^2_1\tmix \ln\Tm\ \|\xi_{u,k} - \xi^{*}_{u,k}\|^2 + \epsap \right).
\end{align*}

Plugging this into~\eqref{e: bias.PG} and~\eqref{e: var.PG}, we obtain
\begin{multline*}
    \left\|\bE_{kh}\big[\hat{q}^{kh}_\MLMC\big] - \nabla_\theta\cL(\theta_k, \lambda_k) \right\|^2
    \\
    = \cO\left( \delta^2_{fg} + \left(1+ \delta^2_{fg}\right)\bigg[ \frac{G^2_1c^2_\gamma\tmix}{\mu^2_\phi\Tm} + G^2_1\|\xi_{u,k} - \xi^*_{u,k}\|^2  + \epsap \bigg]\right)
\end{multline*}
\begin{multline*}
    \left\|\bE_k\big[\hat{q}^{kh}_\MLMC\big] - \nabla_\theta\cL(\theta_k, \lambda_k) \right\|^2
    \\
    = \cO\left( \delta^2_{fg} + \left(1+ \delta^2_{fg}\right)\bigg[ \frac{G^2_1c^2_\gamma\tmix}{\mu^2_\phi\Tm} + G^2_1\|\bE_k[\xi_{u,k}] - \xi^*_{u,k}\|^2  + \epsap \bigg]\right)
\end{multline*}
and
\begin{multline*}
    \bE_{kh}\left\|\hat{q}^{kh}_\MLMC - \nabla_\theta\cL(\theta_k, \lambda_k) \right\|^2 
    \\
    = \tilde{\cO}\left( \sigma^2_{fg} + (1+\sigma^2_{fg})\bigg[\frac{G^2_1c^2_\gamma}{\mu^2_\phi}\tmix + G^2_1\tmix \|\xi_{u,k} - \xi^{*}_{u,k}\|^2 + \epsap \bigg]\right),
\end{multline*}
where we define the scalarization gradient estimators' bias and variance as
\begin{align*}
    \delta^2_{fg} & := \left|\bE_k\big[ \partial_m f(\hat{\bJ}^{\theta_k}_{r,\MLMC})\big] - \partial f_m(\bJ^{\theta_k}_r)\right|^2 + \left|\bE_k\big[\partial g_n(\hat{\bJ}^{\theta_k}_{c,\MLMC})\big] - \partial g_n(\bJ^{\theta_k}_c) \right|^2
    \\
    \sigma^2_{fg} &:= \bE_k\left| \partial_m f(\hat{\bJ}^{\theta_k}_{r,\MLMC}) - \partial f_m(\bJ^{\theta_k}_r)\right|^2 + \bE_k\left|\partial g_n(\hat{\bJ}^{\theta_k}_{c,\MLMC}) - \partial g_n(\bJ^{\theta_k}_c) \right|^2.
\end{align*}

Similarly,~\eqref{e: NPG.naive.single.step} shows the naive estimator for $F(\theta_k)$ as $\frac{1}{T}\sum_{t=0}^{T-1}\hat{F}^k(s^{kh}_t,a^{kh}_t)=\nabla_\theta \ln\pi_{\theta_k}(a|s)\left(\nabla_\theta \ln\pi_{\theta_k}(a|s)\right)^\top.$ From the definition of $F(\theta)$ in~\eqref{e: Fisher.matrix} and Assumption~\ref{a: score.function}, we have
\begin{align*}
    \left\|\bE_{(s,a)\sim \nu^{\theta_k}}\left[\hat{F}^{kh}(s, a)\right] - F(\theta_k)\right\| = 0, \quad \left\|\hat{F}^{kh}(s^{kh}_t, a^{kh}_t) - F(\theta_k) \right\|^2 \leq 4G^4_1,
\end{align*}
respectively. Hence, from Lemma~\ref{app.lem: MLMC}, we have the bias and variance in $\hat{F}^{kh}_\MLMC$ as 
\begin{align*}
    \left\|\bE_{kh}\left[\hat{F}^{kh}_\MLMC(s, a)\right] - F(\theta_k)\right\|^2 & = \cO\left( 4G^4_1\tmix\Tm^{-1} \right),
    \\
    \bE_{kh}\left\|\hat{F}^{kh}_\MLMC(s^{kh}_t, a^{kh}_t) - F(\theta_k) \right\|^2 & =  \cO\left( 4G^4_1\tmix\ln\Tm \right).
\end{align*}
This completes the proof of Lemma~\ref{lem: bias.variance.NPG}. \hfill\qed
 %

\section{Proof of Lemma~\ref{lem: critic.bias.variance}: Analysis of critic estimator}

    %
    We use Lemma~\ref{app.lem: MLMC} to bound the bias and variance of the MLMC estimators $(\hat{q}^{kh}_{u,\MLMC}, \hat{P}^{kh}_{\MLMC}),$ as estimators of $q_u(\theta), P(\theta_k),$ in terms of the bias and variance in corresponding naive estimators.
    
    Recall from~\eqref{e: critic.naive.single.step} that the naive estimators are $(\frac{1}{T}\sum_{t=0}^{T-1}\hat{q}_u(s^{kh}_t, a^{kh}_t), \frac{1}{T}\sum_{t=0}^{T-1}\hat{P}(s^{kh}_t, a^{kh}_t, s^{kh}_{t+1})),$ where 
    \[
        \hat{q}_u(s,a) :=\begin{bmatrix} c_\gamma u(s,a)\\ u(s,a)\phi(s) \end{bmatrix}, \quad \hat{P}(s,a,s') := \begin{bmatrix} c_\gamma  & 0\\ \phi(s) & \phi(s)(\phi(s) - \phi(s')^\top) \end{bmatrix}.
    \]
    
    Using the definition of $q_u(\theta), P(\theta_k)$ in~\eqref{e: critic.estimation.matrix}, we obtain the naive biases as $\|\bE_{\nu^{\theta_k}}[\hat{q}_u] - q_u(\theta_k) \| = 0,$ and $\|\bE_{\nu^{\theta_k}}[\hat{P}] - P(\theta_k) \| = 0.$ While, under Assumption~\ref{a: feature.degenrate}, the naive estimator variance is bounded as
    \[
        \|\hat{q}_u(s^{kh}_t, a^{kh}_t, s^{kh}_{t+1}) - q_u(\theta_k)\|^2 \leq 4(1+c^2_\gamma), \quad \|\hat{P}(s^{kh}_t, a^{kh}_t, s^{kh}_{t+1}) - P(\theta_k)\|^2 \leq 4(c^2_\gamma + 5).
    \]
    Hence, we conclude the following from Lemma~\ref{app.lem: MLMC}:
    \begin{align*}
        \left\|\bE_k\big[\hat{q}^{kh}_{u,\MLMC}\big] - q_u(\theta_k)\right\|^2 = \cO\left(c^2_\gamma\tmix \Tm^{-1} \right),\quad \bE_k\left\|\hat{q}^{kh}_{u,\MLMC} - q_u(\theta_k)\right\|^2 = \cO\left(c^2_\gamma\tmix \ln\Tm \right)
        \\
        \left\|\bE_{kh}\big[\hat{q}^{kh}_{u,\MLMC}\big] - q_u(\theta_k)\right\|^2 = \cO\left(c^2_\gamma\tmix \Tm^{-1} \right),\quad \bE_{kh}\left\|\hat{q}^{kh}_{u,\MLMC} - q_u(\theta_k)\right\|^2 = \cO\left(c^2_\gamma\tmix \ln\Tm \right)
        \\
        \left\|\bE_{kh}\big[\hat{P}^{kh}_{\MLMC}\big] - P(\theta_k)\right\|^2 = \cO\left(c^2_\gamma\tmix \Tm^{-1} \right), \quad \bE_{kh}\left\|\hat{P}^{kh}_{\MLMC} - P(\theta_k)\right\|^2 = \cO\left(c^2_\gamma\tmix \ln\Tm \right).
    \end{align*}
    This completes the proof of Lemma~\ref{lem: critic.bias.variance}. \hfill\qed
    %

\section{Proof of Theorem~\ref{thm: bias.vari.scalar}: Analysis of scalar estimators}

    %
    To prove part statement $(i),$ note that
\begin{align}
    & \bE_k\left[\partial_m f(\hat{\bJ}^{\theta_k}_{r,\MLMC}) \right] 
    \nonumber\\
    & \qquad = \bE_k\left[\partial_m f(\hat{\bJ}^{\theta_k}_{r,1})\right] +  \bE_k\left[\ones_{\{2^{Q_k}\leq\Tm\}}2^{Q_k}\left( \partial_m f(\hat{\bJ}^{\theta_k}_{r,2^{Q_k}}) - \partial_m f(\hat{\bJ}^{\theta_k}_{r,2^{Q_k-1}})  \right)\right]
    \nonumber\\
    & \qquad = \bE_k\left[\partial_m f(\hat{\bJ}^{\theta_k}_{r,1})\right] + \sum_{q=0}^{\floor{\log_2(\Tm)}}2^q\cdot\bP(Q_k=q)\cdot \bE_k\left[ \partial_m f(\hat{\bJ}^{\theta_k}_{r,2^{q}}) - \partial_m f(\hat{\bJ}^{\theta_k}_{r,2^{q-1}})  \right]
    \nonumber\\
    & \qquad \overset{(a)}{=} \bE_k\left[\partial_m f(\hat{\bJ}^{\theta_k}_{r,1})\right] + \sum_{q=0}^{\floor{\log_2(\Tm)}}\bE_k\left[ \partial_m f(\hat{\bJ}^{\theta_k}_{r,2^{q}}) - \partial_m f(\hat{\bJ}^{\theta_k}_{r,2^{q-1}})  \right] 
    \nonumber\\
    & \qquad = \bE_k\left[\partial_m f(\hat{\bJ}^{\theta_k}_{r,2^{\floor{\log_2(\Tm)}}}) \right], 
\end{align}
where $(a)$ follows since $Q_k\sim \textnormal{Geom}(1/2).$ Similar arguments hold for $\partial_ng(\hbJ_{c,\MLMC})$ and $g(\hbJ_{c,\MLMC})$ showing that
\begin{align*}
    \bE_k\left[g(\hat{\bJ}^{\theta_k}_{c,\MLMC}) \right] = \bE_k\left[ g(\hat{\bJ}^{\theta_k}_{c,2^{\floor{\log_2(\Tm)}}}) \right];
    \quad \bE_k\left[\partial_n g(\hat{\bJ}^{\theta_k}_{c,\MLMC}) \right] = \bE_k\left[\partial_n g(\hat{\bJ}^{\theta_k}_{c,2^{\floor{\log_2(\Tm)}}}) \right].
\end{align*}
Further, statement $(ii)$ follows from statement $(i)$ as follows:
\begin{align*}
    \delta^2_{fg} & = \left|\bE_k\big[\partial_m f(\hat{\bJ}^{\theta_k}_{r,2^{\floor{\log_2(\Tm)}}}) \big] - \partial f_m(\bJ^{\theta_k}_r)\right|^2 + \left| \bE_k\big[\partial_n g(\hat{\bJ}^{\theta_k}_{c,2^{\floor{\log_2(\Tm)}}}) \big]  - \partial g_n(\bJ^{\theta_k}_c)\right|^2
    \\
    & \overset{(a)}{\leq } C_f\left\|\bE_k\big[\hat{\bJ}^{\theta_k}_{r,2^{\floor{\log_2(\Tm)}}}\big] - \bJ^{\theta_k}_r\right\|^2 + C_g\left\| \bE_k\big[\hat{\bJ}^{\theta_k}_{c,2^{\floor{\log_2(\Tm)}}} \big]  - \bJ^{\theta_k}_c\right\|^2
    \\
    & \overset{(b)}{=} \cO\left( \frac{(M^2L_f^2+N^2L^2_g)\tmix}{2^{\floor{\log_2(\Tm)}}} \right) = \cO\left( \frac{(M^2L_f^2+N^2L^2_g)\tmix}{\Tm} \right),
\end{align*}
where $(a)$ follows from Assumption~\ref{a: smoothness}, while $(b)$ follows from Lemma~\ref{app.lem: naive.estimator.bounds}

Finally to prove $(iii),$ note that
\begin{align*}
    & \bE_k\left|\partial_m f(\hat{\bJ}^{\theta_k}_{r,\MLMC}) \right|^2 
    \\
    & \qquad \leq  2\bE_k\left|\partial_m f(\hat{\bJ}^{\theta_k}_{r,1})\right|^2 +  2\bE_k\left|\ones_{\{2^{Q_k}\leq\Tm\}}2^{Q_k}\left( \partial_m f(\hat{\bJ}^{\theta_k}_{r,2^{Q_k}}) - \partial_m f(\hat{\bJ}^{\theta_k}_{r,2^{Q_k-1}})  \right)\right|
    \nonumber\\
    & \qquad \leq 2\bE_k\left|\partial_m f(\hat{\bJ}^{\theta_k}_{r,1})\right|^2 +  \sum_{q=0}^{\floor{\log_2(\Tm)}}4^q\cdot\bP(Q_k=q)\cdot 2\bE_k\left| \partial_m f(\hat{\bJ}^{\theta_k}_{r,2^{q}}) - \partial_m f(\hat{\bJ}^{\theta_k}_{r,2^{q-1}})  \right|^2
    \nonumber\\
    & \qquad = 2\bE_k\left|\partial_m f(\hat{\bJ}^{\theta_k}_{r,1})\right|^2 +  \sum_{q=0}^{\floor{\log_2(\Tm)}}2^{q+1}\bE_k\left| \partial_m f(\hat{\bJ}^{\theta_k}_{r,2^{q}}) - \partial_m f(\hat{\bJ}^{\theta_k}_{r,2^{q-1}})  \right|^2
    \nonumber\\
    & \qquad \leq 2\bE_k\left|\partial_m f(\hat{\bJ}^{\theta_k}_{r,1})\right|^2 + \sum_{q=0}^{\floor{\log_2(\Tm)}}2^{q+1}L^2_f\bE_k\left\| \hat{\bJ}^{\theta_k}_{r,2^{q}} - \hat{\bJ}^{\theta_k}_{r,2^{q-1}}  \right\|^2
    \nonumber\\
    & \qquad \leq 2\bE_k\left|\partial_m f(\hat{\bJ}^{\theta_k}_{r,1})\right|^2 + \sum_{q=0}^{\floor{\log_2(\Tm)}}2^{q+2}L^2_f\left(\bE_k\left\| \hat{\bJ}^{\theta_k}_{r,2^{q}} - \bJ^{\theta_k}_r  \right\|^2 + \bE_k\left\| \bJ^{\theta_k}_r - \hat{\bJ}^{\theta_k}_{r,2^{q-1}}  \right\|^2 \right)
    \\
    & \qquad \overset{(a)}{\leq} 2C^2 + \sum_{q=0}^{\floor{\log_2(\Tm)}}2^{q+2}L^2_fC_1\tmix\cO\left( \frac{1}{2^q} + \frac{1}{2^{q-1}}\right)
    \\
    & \qquad = 2C^2 + \sum_{q=0}^{\floor{\log_2(\Tm)}}2^{q+2}L^2_fC_1\tmix\cO\left( \frac{1}{2^{q-1}}\right) 
    \\
    & \qquad = 2C^2 + \cO(L^2_fC_1\tmix)\sum_{q=0}^{\floor{\log_2\Tm}}8 = \cO(L^2_f\tmix \ln\Tm),
\end{align*}
where $(a)$ follows from Lemma~\ref{app.lem: naive.estimator.bounds}.

Similar arguments hold for $\partial_n g(\hbJ_{c,\MLMC}).$ Consequently, we have
\begin{align*}
    \sigma^2_{fg} & =  \bE_k\left| \partial_m f(\hat{\bJ}^{\theta_k}_{r,\MLMC}) - \partial f_m(\bJ^{\theta}_r)\right|^2 + \bE_k\left|\partial g_n(\hat{\bJ}^{\theta}_{c,\MLMC}) - \partial g_n(\bJ^{\theta}_c) \right|^2.
    \\
    &= \cO\left(M^2L_f^2+N^2L^2_g)\tmix \log_22^{\floor{\log_2(\Tm)}} \right) = \cO\left( (M^2L_f^2+N^2L^2_g)\tmix\ln\Tm \right).
\end{align*}
This completes the proof of Theorem~\ref{thm: bias.vari.scalar}. \hfill\qed
%

\section{Proof of Theorem~\ref{thm: main result}}
\label{appendix: proof.main.result}

Combining the scalar gradient bias and variance bounds from Theorem~\ref{thm: bias.vari.scalar} with the critic error from Theorem~\ref{thm: critic.bounds}, and plugging into Theorem~\ref{thm: NPG.error}, we have
\begin{align*}
    &\left\|\bE_k[w_{k}] - \wS_{k}\right\|^2 
    \\
    & \quad = \tilde{\cO}\bigg( \frac{1}{T^2} +\frac{\tmix}{T} + \left(1+ \frac{\tmix}{\Tm}\right)\bigg[ \frac{\tmix}{\Tm} + \frac{\tmix^2}{\Tm}  + \epsap \bigg]\bigg)
    = \tilde{\cO}\left(  \frac{\tmix^2}{\Tm} +  \epsap \right);
    \\
    & \bE_k\left\|w_{u,k} - \wS_{u,k}\right\|^2
    \\
    & \quad =  \tilde{\cO}\bigg( \frac{1}{T^2}+ \frac{\tmix}{H}+\bigg(\frac{\tmix}{\Tm} + \frac{\ln \Tm}{H^2}\bigg) 
    + \bigg(2 + \frac{\tmix}{\Tm} + \frac{\tmix}{T}\bigg)\left[\frac{\tmix^2}{H} + \epsap \right] \bigg)
    = \tilde{\cO}\left(\frac{\tmix^2}{H} + \epsap \right).
\end{align*}

Now, if we plug the above bounds into Lemma~\ref{lem: gen.framework} and set $\Tm=T,$ we get
 \begin{align*}
        & \frac{1}{K}\sum_{k=0}^{K-1}\bE\big[\cL(\piS, \lambda_k) - \cL(\theta_k,\lambda_k) \big] 
        \\
        & \quad \leq \sqrt{\epsb} +  \frac{G_1}{K}\sum_{k=0}^{K-1}\tilde{\cO}\left(\frac{\tmix}{\sqrt{T}} + \sqrt{\epsap} \right) + \frac{\alpha}{K}\sum_{k=0}^{K-1}\tilde{\cO}\left(\frac{\tmix^2}{H} + \epsap \right)
        + \cO\left(\frac{1}{K}+\frac{1}{\alpha K}\right)
        \\
        & \quad = \tilde{\cO}\left(\sqrt{\epsb} + \sqrt{\epsap} + \frac{\tmix^2}{\sqrt{T}} + \alpha + \frac{1}{\alpha K} \right).
\end{align*}

This gives us the final Lagrange error. Now, to obtain the objective error, recall from~\eqref{e: lagrange.function} that $ \left[ f(\bJ^{\piS}_r) - f(\bJ^{\theta_k}_r)\right] = \big[\cL(\piS, \lambda_k) - \cL(\theta_k,\lambda_k) \big] -\lambda_k\big[ g(\bJ^{\piS}_c) - g(\bJ^{\theta_k}_c)\big].$ Hence,
\begin{align}\label{e: proof.object.error}
    & \frac{1}{K}\sum_{k=0}^{K-1}\bE\left[ f(\bJ^{\piS}_r) - f(\bJ^{\theta_k}_r)\right]
    \nonumber\\
    & \quad = \tilde{\cO}\left(\sqrt{\epsb} + \sqrt{\epsap} + \frac{\tmix^2}{\sqrt{T}} + \alpha + \frac{1}{\alpha K}  \right) - \frac{1}{K}\sum_{k=0}^{K-1} \bE\left[\lambda_k\big[ g(\bJ^{\piS}_c) - g(\bJ^{\theta_k}_c)\big]\right].
\end{align}

To bound the second term, note that for $K>0,$
\begin{align*}
    0\leq \lambda^2_K & = \sum_{k=0}^{K-1}\left( \lambda^2_{k+1} - \lambda^2_k \right) \overset{(a)}{=} \sum_{k=0}^{K-1}\left( \left(\Pi_{[0,2/\delta]}\left(\lambda_k - \beta g_k \right)\right)^2 - \lambda^2_k\right)
    \\
    & \leq \sum_{k=0}^{K-1}\left( \left(\lambda_k - \beta g_k \right)^2 - \lambda^2_k\right)
    = \beta^2\sum_{k=0}^{K-1}g^2_k -2\beta\sum_{k=0}^{K-1}\lambda_k g_k 
    \\
    & \overset{(b)}{\leq} \beta^2\sum_{k=0}^{K-1}g^2_k +  2\beta \sum_{k=0}^{K-1}\lambda_k\left(g(\bJ^{\theta_k}_c) - g_k\right) + 2\beta\sum_{k=0}^{K-1}\lambda_k \big[g(\bJ^{\piS}_c) - g(\bJ^{\theta_k}_c)\big]
    \\
    & \overset{(c)}{\leq} \frac{4K\beta^2}{\delta^2} +  \frac{4\beta}{\delta} \sum_{k=0}^{K-1}\left(g(\bJ^{\theta_k}_c) - g_k\right) + 2\beta\sum_{k=0}^{K-1}\lambda_k \big[g(\bJ^{\piS}_c) - g(\bJ^{\theta_k}_c)\big]
\end{align*}
where $(a)$ follows from the dual update rule in~\eqref{e: actor.update}, $(b)$ uses the fact that $g(\bJ^{\piS}_c)\geq 0$ since $\piS$ is a feasible solution to~\eqref{e: lagrange.function}, and $(c)$ follows since $g_k, \lambda_k \leq 2/\delta,$

Taking expectation and rearranging terms gives us
\begin{align}\label{e: proof.dual.error}
    & -\frac{1}{K}\sum_{k=0}^{K-1}\lambda_k \big[g(\bJ^{\piS}_c) - g(\bJ^{\theta_k}_c)\big] \leq \frac{2}{\delta K}\sum_{k=0}^{K-1}\bE\left[ g(\bJ^{\theta_k}_c) - g_k\right] + \frac{2\beta}{\delta^2}
    \nonumber\\
    & = \frac{2}{\delta K}\sum_{k=0}^{K-1}\bE\left[ g(\bJ^{\theta_k}_c) - g(\hbJ_{c,\MLMC})\right] + \frac{2\beta}{\delta^2} 
    \nonumber\\
    & \leq \frac{2}{\delta K}\sum_{k=0}^{K-1}\bE\left|g(\bJ^{\theta_k}_c) - \bE_k\big[g(\hbJ_{c,\MLMC})\big]\right| + \frac{2\beta}{\delta^2}
    \overset{(a)}{=} \tilde{\cO}\left( \frac{\tmix}{\sqrt{\Tm}} + \beta \right) = \tilde{\cO}\left( \frac{\tmix}{\sqrt{T}} + \beta \right),
\end{align}
where $(a)$ follows from Theorem~\ref{thm: bias.vari.scalar}. Substituting \eqref{e: proof.dual.error} into \eqref{e: proof.object.error}, we have the final global convergence rate
\begin{align*}
    & \frac{1}{K}\sum_{k=0}^{K-1}\bE\left[ f(\bJ^{\piS}_r) - f(\bJ^{\theta_k}_r)\right]
    = \tilde{\cO}\left(\sqrt{\epsb} + \sqrt{\epsap} + \frac{\tmix^2}{\sqrt{T}} + \alpha + \frac{1}{\alpha K} +\beta \right).
\end{align*}
Taking $K=T$ and $\alpha=\beta=1/\sqrt{T}$ gives the desired global convergence rate.

At last, we need to bound the expected constraint violation rate. Since $g(\bJ^{\piS}_c)\geq 0,$  we have
\begin{align}\label{e: final.converge}
    f(\bJ^{\piS}_r) &  - \frac{1}{K}\sum_{k=0}^{K-1}\bE\big[f(\bJ^{\theta_k}_r) \big] + \frac{2}{\delta K}\sum_{k=0}^{K-1}\bE\big[-g(\bJ^{\theta_k}_c) \big] 
    \nonumber\\
    & \leq \frac{1}{K}\sum_{k=0}^{K-1}\bE[\cL(\piS,\lambda_k) - \cL(\theta_k, \lambda_k)] + \frac{1}{K}\sum_{k=0}^{K-1} \bE\left[g(\bJ^{\theta_k}_c)\left( \lambda_k - \frac{2}{\delta} \right)\right]
    \nonumber\\
    & = \tilde{\cO}\left(\sqrt{\epsb} + \sqrt{\epsap} + \frac{\tmix^2}{\sqrt{T}} + \alpha + \frac{1}{\alpha K} +\beta \right) + \frac{1}{K}\sum_{k=0}^{K-1} \bE\left[g(\bJ^{\theta_k}_c)\left( \lambda_k - \frac{2}{\delta} \right)\right]
\end{align}
To bound the last term, note that the dual update rule in~\eqref{e: actor.update} gives us
\begin{align*}
    \left| \lambda_{k+1} - \frac{2}{\delta}\right|^2 \overset{(a)}{\leq} \left|\lambda_k - \beta g_k -\frac{2}{\delta} \right|^2 = \left| \lambda_k -\frac{2}{\delta} \right|^2 -2\beta g_k\left(\lambda_k - \frac{2}{\delta} \right) + \beta^2g^2_k
\end{align*}
where $(a)$ uses the non-expansiveness of the projection $\Pi_{[0,2/\delta]}.$ Now, summing over $k$ in $\{0,\ldots,K-1\}$ and dividing by $\beta K$ gives us
\begin{align*}
    0 \leq \frac{1}{\beta K}\left| \lambda_K - \frac{2}{\delta}\right|^2 \leq \frac{1}{\beta K}\left| \lambda_0 - \frac{2}{\delta}\right|^2 -\frac{2\beta}{\beta K}\sum_{k=0}^{K-1}g_k\left(\lambda_k - \frac{2}{\delta}\right)  + \tilde{\cO}(\beta)  
\end{align*}
Hence, taking expectation on both sides gives us
\begin{align*}
    \frac{2}{K}\sum_{k=0}^{K-1}\bE\left[g(\bJ^{\theta_k}_c)\left(\lambda_k - \frac{2}{\delta}\right) \right] & \leq \tilde{O}\left(\beta + \frac{1}{\beta K}\right) + \frac{4}{\delta K}\sum_{k=0}^{K-1}\left|\bE\left[g(\bJ^{\theta_k}_c) - g_k\right]\right|
    \\
    & \overset{(a)}{=} \tilde{O}\left(\beta + \frac{1}{\beta K}\right) + \frac{4}{\delta K}\sum_{k=0}^{K-1}\bE\left|g(\bJ^{\theta_k}_c) - \bE_k g_k\right|
    \\
    & \overset{(b)}{=} \tilde{O}\left(\beta + \frac{1}{\beta K}\right) + \frac{4}{\delta K}\sum_{k=0}^{K-1}\bE\left|g(\bJ^{\theta_k}_c) - \bE_k g(\hbJ_{c,\MLMC})\right|
    \\
    & \overset{(c)}{=} \tilde{O}\left(\beta + \frac{1}{\beta K} + \frac{\tmix}{\sqrt{\Tm}}\right)
    = \tilde{O}\left(\beta + \frac{1}{\beta K} + \frac{\tmix}{\sqrt{T}}\right),
\end{align*}

where $(a)$ follows from taking conditional expectation and using Jensen's inequality, $(b)$ uses the definition of $g_k,$ and $(c)$ follows from Theorem~\ref{thm: bias.vari.scalar}.

Substituting this into~\eqref{e: final.converge} gives 
\begin{multline*}
    f(\bJ^{\piS}_r)  - \frac{1}{K}\sum_{k=0}^{K-1}\bE\big[f(\bJ^{\theta_k}_r) \big] + \frac{2}{\delta K}\sum_{k=0}^{K-1}\bE\big[-g(\bJ^{\theta_k}_c) \big] 
    \\
    = \tilde{\cO}\left(\sqrt{\epsb} + \sqrt{\epsap} + \frac{\tmix^2}{\sqrt{T}} + \alpha + \frac{1}{\alpha K} +\beta + \frac{1}{\beta K}\right).
\end{multline*}
Now, invoking~\citep[Lemma~G.6]{xu2026global}, we conclude
\begin{align*}
    \sum_{k=0}^{K-1}\bE\big[-g(\bJ^{\theta_k}_c) \big] 
    = \delta\cdot \tilde{\cO}\left(\sqrt{\epsb} + \sqrt{\epsap} + \frac{\tmix^2}{\sqrt{T}} + \alpha + \beta +  \frac{1}{\alpha K} + \frac{1}{\beta K} \right).
\end{align*}
Setting $K=T$ and $\alpha=\beta=1/\sqrt{T}$ completes the proof of Theorem~\ref{thm: main result}. \hfill$\qed$

\section{Technical lemmas}
\label{appendix: technical.lemmas}

\begin{lemma}[Performance difference lemma for concave utility]\label{app.lem: performance.diff}
    For any smooth scalarization  $\varphi:\bR^D\to\bR,$ multi-objective reward/cost $u:\cS\times\cA\to\bR^D,$ and policies $\pi$ and $\pi',$ we have
    \[        
        \varphi(\bJ^\pi_u) - \varphi(\bJ^{\pi'}_u) \leq \sum_{d=1}^D \partial_d \varphi(\bJ^{\pi'}) \bE_{(s,a)\sim \nu^{\pi}}[A_{u_d}^{\pi'}(s,a)],
    \]
    where $A^{\pi'}_{c_d}$ is the advantage function of policy $\pi'$ associated with objective $d\in[D],$ as defined in~\eqref{e: advantage.func}.
\end{lemma}
\begin{proof}
    Since $\varphi$ is concave and differentiable, we have
    \[
        \varphi(\bJ^\pi) - \varphi(\bJ^{\pi'}) \leq  \left\langle \nabla \varphi(\bJ^{\pi'}), (\bJ^{\pi} - \bJ^{\pi'}) \right\rangle = \sum_{d=1}^{D} \partial_d \varphi (\bJ^{\pi'})(J^{\pi}_{u_d} - J^{\pi'}_{u_d}).
    \]
    The claim follows from the Performance difference lemma~\cite[Lemma 4]{bai2024regret}, which states that for every objective $d,$ we have $J^\pi_{u_d} - J^{\pi'}_{u_d} = \bE_{(s,a)\sim \nu^\pi}\big[A^{\pi'}_{u_d}(s,a)\big].$ 
\end{proof}

\begin{lemma}[{\cite[Lemma B.2]{xu2026global}}]\label{app.lem: MLMC}
    Let $(Z_t)$ be a time-homogeneous ergodic Markov chain with a unique invariant distribution $d_Z,$ and a mixing-time $\tmix.$ Assume that $F(x,Z)$ is an estimate of $F(x).$ Let for all $t\geq 0,$
    \[
        \|\bE_{d_Z}[F(x,Z)] - F(x)\|^2\leq \delta^2 \quad \text{and} \quad  \bE\| F(x,Z_t) - \bE_{d_Z}[F(x,Z)]\|^2 \leq \sigma^2.
    \]
    Further, let $Q\sim \textnormal{Geom(1/2)}.$ Then, the MLMC estimator defined as 
    \[
        g_\MLMC := g_1 + \ones_{2^Q\leq \Tm}\cdot 2^Q\left( g_{2^Q} - g_{2^{Q-1}} \right), \quad \text{where} \quad g_j:= j^{-1}\sum_{t=0}^{j-1}F(x,Z_t),
    \]
    satisfies the following bounds.
    \begin{enumerate}[noitemsep, label = (\alph*)]
        \item $\bE[g_\MLMC] = \bE[g_{\floor{\log_2\Tm}}];$
        \item $\bE\| F(x) - g_\MLMC\|^2 = \cO\left( \sigma^2\tmix\log_2\Tm + \delta^2 \right);$
        \item $\|F(x) - \bE[g_\MLMC]\|^2 = \cO\left( \sigma^2\tmix \Tm^{-1} + \delta^2 \right).$
    \end{enumerate}

\end{lemma}

\begin{lemma}[{\cite[Lemma B.3]{xu2026global}}]\label{app.lem: naive.estimator.bounds}
    Consider the setup used in Lemma~\ref{app.lem: MLMC}. Then the following holds for all $T>0:$
    \[
        \bE\left\|\frac{1}{T}\sum_{t=0}^{T-1}F(x,Z_t) - \bE_{d_Z}\left[F(x,Z) \right] \right\|^2 \leq \frac{C_1\tmix}{T}\sigma^2,
    \]
    where $C_1:=16\left(1 + 1/\ln^24 \right).$
\end{lemma}


\end{document}